%% file: arxiv.tex
\documentclass{article}
\PassOptionsToPackage{numbers,sort&compress}{natbib}

\usepackage[utf8]{inputenc} 
\usepackage[T1]{fontenc}    
\usepackage{hyperref}       
\usepackage{url}            
\usepackage{booktabs}       
\usepackage{amsmath, amssymb, amsthm, amsfonts}       
\usepackage{nicefrac, xfrac}       
\usepackage{microtype}      
\usepackage{xcolor}         
\usepackage{bm}
\usepackage{enumitem}
\usepackage{graphicx}
\usepackage{algorithm}
\usepackage{algorithmic}
\usepackage{natbib}[numbers,sort&compress]
\usepackage{geometry}
\usepackage{authblk,textcomp}
\usepackage{mathtools}

\usepackage[cal=euler]{mathalfa}
\usepackage{libertine}
\usepackage[capitalize,noabbrev]{cleveref}
\usepackage{mathrsfs}

\usepackage{enumitem}
\newlist{condenum}{enumerate}{1} 
\setlist[condenum]{label=(\roman*), ref=(\roman*)}
\crefname{condenumi}{Assumption}{Assumptions}

\theoremstyle{plain}
\newtheorem{theorem}{Theorem}[section]
\newtheorem{proposition}[theorem]{Proposition}
\newtheorem{lemma}[theorem]{Lemma}
\newtheorem{corollary}[theorem]{Corollary}
\newtheorem{Conjecture}[theorem]{Conjecture}
\newtheorem{remark}[theorem]{Remark}
\theoremstyle{definition}
\newtheorem{definition}[theorem]{Definition}
\newtheorem{assumption}[theorem]{Assumption}
\crefname{assumption}{assumption}{assumptions}
\Crefname{equation}{Eq.}{Eqs.}
\crefname{equation}{eq.}{eqs.}

\input{macros}

\hypersetup{pdfauthor={},pdftitle={},%
            colorlinks, linktocpage=true, pdfstartpage=1, pdfstartview=FitV,%
    breaklinks=true, pdfpagemode=UseNone, pageanchor=true, pdfpagemode=UseOutlines,%
    plainpages=false, bookmarksnumbered, bookmarksopen=true, bookmarksopenlevel=1,%
    hypertexnames=true, pdfhighlight=/O,%
    urlcolor=orange, linkcolor=blue, citecolor=blue
        }

\title{Asymptotics of Learning with Deep Structured (Random) Features}

\author[1*]{Dominik Schr\"oder}
\author[2*]{Daniil Dmitriev}
\author[3*]{Hugo Cui}
\author[4]{Bruno Loureiro}
\affil[1]{\small Department of Mathematics, ETH Zurich, 8006 Z\"urich, Switzerland}
\affil[2]{\small Department of Mathematics, ETH Zurich and ETH AI Center, 8092 Z\"urich, Switzerland}
\affil[3]{\small Statistical Physics Of Computation lab.,
Institute of Physics, \'Ecole Polytechnique F\'ed\'erale de Lausanne (EPFL), \newline 1015 Lausanne, Switzerland}
\affil[4]{\small D\'epartement d'Informatique, \'Ecole Normale Sup\'erieure (ENS) - PSL \& CNRS, 
F-75230 Paris cedex 05, France}
\affil[ ]{\textit {dschroeder@ethz.ch, daniil.dmitriev@ai.ethz.ch, hugo.cui@epfl.ch, bruno.loureiro@di.ens.fr}}

\affil[ *]{\textit {Main contributions}}
\makeatletter
\newtheorem*{rep@theorem}{\rep@title}
\newcommand{\newreptheorem}[2]{%
\newenvironment{rep#1}[1]{%
 \def\rep@title{#2 \ref{##1}}%
 \begin{rep@theorem}}%
 {\end{rep@theorem}}}
\makeatother

\theoremstyle{plain}
\numberwithin{theorem}{section}
\newreptheorem{theorem}{Theorem}

\theoremstyle{remark}
\date{\today}

\begin{document}
\maketitle

\begin{abstract}
For a large class of feature maps we provide a tight asymptotic characterisation of the test error associated with learning the readout layer, in the high-dimensional limit where the input dimension, hidden layer widths, and number of training samples are proportionally large. This characterization is formulated in terms of the population covariance of the features. Our work is partially motivated by the problem of learning with Gaussian rainbow neural networks, namely deep non-linear fully-connected networks with random but structured weights, whose row-wise covariances are further allowed to depend on the weights of previous layers. For such networks we also derive a closed-form formula for the feature covariance in terms of the weight matrices. We further find that in some cases our results can capture feature maps learned by deep, finite-width neural networks trained under gradient descent.
\end{abstract}

\section{Introduction}
\label{Introduction}
\input{sections/introduction}
\section{Setting}
\label{sec:setting}
\input{sections/setting}
\section{Test error of Lipschitz feature models}
\label{sec gen err}
\input{sections/error}
\section{Population covariance for rainbow networks}
\label{sec:linearization}
\input{sections/linearization}

\section{Linearizing trained neural networks}
\label{section lin grad}
\input{sections/trained}
\section*{Acknowledgements}
\input{sections/acknowledgements}

\bibliographystyle{unsrt}
\bibliography{biblio}
\newpage
\appendix
\section{Anisotropic asymptotic equivalents}
\label{appendix gen error}
\input{sections/appendix/gen_error}

\section{Linearization of population covariance}
\label{linearization}
\input{sections/appendix/app_linearization}
\section{Details on numerics}
\label{app: numerics}
\input{sections/appendix/phenomenology}

\end{document}

%% file: macros.tex
\newcommand{\Card}[1]{\lvert #1 \rvert}
\newcommand{\R}{\mathbf R}

\newcommand{\cH}{\mathcal H}

\newcommand{\cN}{\mathcal N}

\newcommand{\sS}{\mathscr S}

\newcommand{\x}{\mathsf x}
\newcommand{\y}{\mathsf y}
\newcommand{\di}{d}
\newcommand{\n}{n}
\newcommand{\p}{p}
\newcommand{\w}{\theta}
\newcommand{\W}{W}
\newcommand{\f}{\varphi}

\newcommand{\fast}{\varphi_\ast}
\newcommand{\depth}{L}
\newcommand{\gen}{\rm gen}
\newcommand{\noise}{\varepsilon}

\newcommand{\N}{\mathbb N}

\DeclareFontFamily{U}{mathx}{}
\DeclareFontShape{U}{mathx}{m}{n}{<-> mathx10}{}
\DeclareSymbolFont{mathx}{U}{mathx}{m}{n}
\DeclareMathAccent{\wh}{0}{mathx}{"70}
\DeclareMathAccent{\wc}{0}{mathx}{"71}
\DeclareMathAccent{\wt}{0}{mathx}{"72}

\DeclareMathOperator{\E}{\mathbf E} 
\DeclareMathOperator{\Var}{Var}

\DeclareMathOperator{\rank}{rank}

\DeclareMathOperator{\dif}{d}
\DeclareMathOperator{\Tr}{Tr}

\DeclarePairedDelimiter{\norm}{\lVert}{\rVert}
\DeclarePairedDelimiter{\abs}{\lvert}{\rvert}
\DeclarePairedDelimiter{\braket}{\langle}{\rangle}
\DeclarePairedDelimiterX\set[1]\{\}{
    
    #1
}
\DeclarePairedDelimiterXPP\bigO[1]{O}{(}{)}{}{#1}
 \DeclarePairedDelimiterXPP\Prob[1]{\mathbb{P}}(){}{
   
   #1}

%% file: sections/introduction.tex
Deep neural networks are the backbone of most successful machine learning algorithms in the past decade. Despite their ubiquity, a firm theoretical understanding of the very basic mechanism behind their capacity to adapt to different types of data and generalise across different tasks remains, to a large extent, elusive. For instance, what is the relationship between the inductive bias introduced by the network architecture and the representations learned from the data, and how does it correlate with generalisation? Despite the lack of a complete picture, insights can be found in recent empirical and theoretical works.

On the theoretical side, a substantial fraction of the literature has focused on the study of deep networks at initialisation, motivated by the lazy training regime of large-width networks with standard scaling. Besides the mathematical convenience, the study of random networks at initialisation have proven to be a valuable theoretical testbed -- allowing in particular to capture some empirically observed behaviour, such as the double-decent \cite{belkin2019reconciling} and benign overfitting \cite{Bartlett20} phenomena. As such, proxys for networks at initialisation, such as the Random Features (RF) model \cite{Rahimi2007RandomFF} have thus been the object of considerable theoretical attention, with their learning being asymptotically characterized in the two-layer case \cite{Goldt2021TheGE, Goldt2020ModellingTI, Gerace2020GeneralisationEI, Hu2020UniversalityLF, Dhifallah2020, Mei2019TheGE, Mei2021GeneralizationEO} and the deep case \cite{ZavatoneVeth2022ContrastingRA, schroder2023deterministic, bosch2023precise, zavatone2023learning}. With the exception of \cite{Gerace2020GeneralisationEI, mel2022anisotropic} (limited to two-layer networks) and \cite{zavatone2023learning} (limited to linear networks), all the analyses for non-linear deep RFs assume unstructured random weights. In sharp contrast, the weights of trained neural networks are fundamentally structured - restricting the scope of these results to networks at initialization.

Indeed, an active research direction consists of empirically investigating how the statistics of the weights in trained neural networks encode the learned information, and how this translates to properties of the predictor, such as inductive biases \cite{Thamm_2022, JMLR:v22:20-410}. Of particular relevance to our work is a recent observation by \cite{guth2023rainbow} that a random (but structured) network with the weights sampled from an ensemble with matching statistics can retain a comparable performance to the original trained neural networks. In particular, for some tasks it was shown that second order statistics suffices -- defining a Gaussian \emph{rainbow network} ensemble.

Our goal in this manuscript is to provide an exact asymptotic characterization of the properties of \emph{Gaussian rainbow networks}, i.e. deep, non-linear networks with structured random weights. Our \textbf{main contributions} are:
\begin{itemize}
    \item We derive a tight asymptotic characterization of the test error achieved by performing ridge regression with Lipschitz-continuous feature maps, in the high-dimensional limit where the dimension of the features and the number of samples grow at proportional rate. This class of feature maps encompasses as a particular case Gaussian rainbow network features.
    \item The asymptotic characterization is formulated in terms of the population covariance of the features. For Gaussian rainbow networks, we explicit a closed-form expression of this covariance, formulated as in the unstructured case \cite{schroder2023deterministic} as a simple linear recursion depending on the weight matrices of each layer. These formulae extend similar results of \cite{Cui2023, schroder2023deterministic} for independent and unstructured weights to the case of structured --and potentially correlated-- weights.
    \item We empirically find that our theoretical characterization captures well the learning curves of some networks trained by gradient descent in the lazy regime.
\end{itemize}
\paragraph{Code --} The code for the numerical experiments described in~\Cref{app: numerics} is openly available in  \href{https://github.com/wirhabenzeit/feature-ridge-regression}{this repository.}

\subsection*{Related works}
\paragraph{Random features} (Rfs) were introduced in \cite{Rahimi2007RandomFF} as a computationally efficient way of approximating large kernel matrices. In the shallow case, the asymptotic spectral density of the conjugate kernel was derived in \cite{Liao2018, Pennington2019NonlinearRM, Benigni2021}. The test error was on the other hand characterized in \cite{Mei2019TheGE, Mei2021GeneralizationEO} for ridge regression, and extended to generic convex losses by \cite{Gerace2020GeneralisationEI, Goldt2021TheGE, Dhifallah2020}, and in \cite{Sur2020, Loureiro2021CapturingTL, Bosch2022} for other penalties. RFs have been studied as a model for networks in the lazy regime, see e.g. \cite{Ghorbani2019, Ghorbani2020WhenDN, NEURIPS2019_5481b2f3, pmlr-v139-refinetti21b}. The role of structure in the RF weights was discussed in \cite{Gerace2022} for rotationally invariant weights and \cite{ mel2022anisotropic} for anisotropic Gaussian weights.
\paragraph{Deep RFs -- } Recent work have addressed the problem of extending these results to deeper architectures. In the case of linear networks, a sharp characterization of the test error is provided in \cite{ZavatoneVeth2022ContrastingRA} for the case of unstructured weights and \cite{zavatone2023learning} in the case of structured weights. For non-linear RFs, \cite{schroder2023deterministic} provides deterministic equivalents for the sample covariance matrices, and \cite{schroder2023deterministic, bosch2023precise} provide a tight characterization of the test error. The recent work of \cite{guth2023rainbow} provides empirical evidence that for a given trained neural network, a resampled network from an ensemble with matching statistics (\emph{rainbow networks}) might achieve comparable generalization performance, thereby partly bridging the gap between random networks and trained networks.

%% file: sections/setting.tex
Consider a supervised learning task with training data $(\mathsf x_{i},\mathsf y_{i})_{i\in[n]}$. In this manuscript, we are interested in studying the statistics of linear predictors $f_{\w}(\x)	= \frac{1}{\sqrt{\p}}\w^{\top}\f(\x)$ for a class of fixed feature maps $\f:\R^{\di}\to\R^{\p}$ and weights $\w\in\R^{\p}$ trained via empirical risk minimization:
\begin{align}
    \label{eq:def:erm}
    \hat{\w}_{\lambda} = \underset{\w\in\R^{\p}}{\min}~\sum\limits_{i\in[\n]}\left(\y_{i}-f_{\w}(\x_{i})\right)^{2} + \lambda||\w||^{2}.
\end{align}
Of particular interest is the generalization error:
\begin{align}\label{eq:gen err def}
    \mathcal{E}_{\gen}(\hat{\w}_{\lambda}) = \E\left(\y - f_{\hat{\w}_{\lambda}}(\x)\right)^{2}
\end{align}
where the expectation is over a fresh sample from the same distribution as the training data. More precisely, our results will hold under the following assumptions.

\begin{assumption}[Labels]\label{ass:labels}
    We assume that the labels $y_i$ are generated by another feature map $\fast\colon\R^d\to\R^k$ as
    \begin{equation}
        \y_i = \frac{1}{\sqrt{k}}\theta_\ast^\top \fast(\x_i) + \noise_i,
    \end{equation}
    where $\noise\in\R^n$ is an additive noise vector (independent of the covariates $\x_i$) of zero mean and covariance $\Sigma:=\E \noise\noise^\top$, and $\theta_\ast\in\R^{k}$ is a deterministic weight vector.
\end{assumption}
\begin{assumption}[Data \& Features]\label{ass:data+features}
    We assume that the covariates $\x_i$ are independent and come from a distribution such that
    \begin{condenum}
        \item\label{cond centered} the feature maps $\f,\fast$ are centered\footnote{This is a commonly used assumption which simplifies the analysis. Our techniques also apply to the case of non-zero mean, however doing so would add a rank-one component to the sample covariance matrix, considerably complicating the final expressions for the deterministic equivalents.} in the sense $\E\f(\x_i)=0$, $\E\fast(\x_i)=0$,
        \item\label{cond cov} the feature covariances
        \begin{equation}
            \begin{split}
                \Omega&:=\E \f(\x_i)\f(\x_i)^\top\in\R^{p\times p},\\
                \Psi&:=\E \fast(\x_i)\fast(\x_i)^\top\in\R^{k\times k},\\
                \Phi&:=\E \f(\x_i)\fast(\x_i)^\top\in\R^{p\times k},
            \end{split}
        \end{equation}
        have uniformly bounded spectral norm.
        \item\label{cond Lipschitz conc} scalar Lipschitz functions of the feature matrices
        \begin{equation}
            \begin{split}
                X&:=(\f(\x_1),\ldots,\f(\x_n))\in\R^{p\times n}\\
                Z&:=(\fast(\x_1),\ldots,\fast(\x_n))\in\R^{k\times n}
            \end{split}
        \end{equation} are uniformly sub-Gaussian.
    \end{condenum}
\end{assumption}
\begin{assumption}[Proportional regime]\label{ass:dimensions}
    The number of samples $n$ and the feature dimensions $p,k$ are all large and comparable, see~\Cref{thm genRMT informal} later.
\end{assumption}
\begin{remark}\label{remark suff cond}
    We formulated~\Cref{ass:data+features} as a joint assumption on the covariates distribution and the feature maps. A conceptually simpler but less general condition would be to assume that
    \begin{condenum}[topsep=0pt,itemsep=-1ex,partopsep=1ex,parsep=1ex]
        \item[(ii')] the covariates $\x_i$ are Gaussian with bounded covariance $\Omega_0:=\E\x_i\x_i^\top$
        \item[(iii')] the feature maps $\f,\fast$ are Lipschitz-continuous
    \end{condenum}
    instead of~\Cref{cond cov,cond Lipschitz conc}.
\end{remark}
The setting above defines a quite broad class of problems, and the results that follow in Section~\ref{sec gen err} will hold under these generic assumptions. The main class of feature maps we are interested in are \emph{deep structured feature models}.
\begin{definition}[Deep structured feature model]\label{deep feature model}
    For any fixed $L\in\N$ and dimensions $d,p_1,\ldots,p_L=p$, let $\f_1,\ldots,\f_L\colon\R\to\R$ be Lipschitz-continuous \emph{activation functions} $\abs{\f_l(a)-\f_l(b)}\lesssim\abs{a-b}$ applied entrywise, and let $W_1\in\R^{p_1\times d}, W_2\in\R^{p_2\times p_1},\ldots$ be deterministic \emph{weight matrices} with uniformly bounded spectral norms, $\norm{W_l}\lesssim 1$. We then call
\begin{align}\label{deep feature eq}
    \f(\x) := \f_{L}\left(\W_{L}\f_{L-1}\left(\cdots W_{2}\f_{1}\left(\W_{1}\x\right)\right)\right).
\end{align}
a \emph{deep structured feature} model. 
\end{definition}
Note that~\cref{deep feature eq} defines a Lipschitz-continuous map\footnote{$\norm{\f(W\x)-\f(W\x')}^2 = \sum_{i} \abs{\f(w_i^\top \x)-\f(w_i^\top \x')}^2 \lesssim \sum_i \abs{w_i^\top(\x-\x')}^2 = \norm{W(\x-\x')}^2\lesssim \norm{\x-\x'}^2$} $\f\colon\R^d\to\R^{p},\fast\colon\R^d\to\R^k$ and therefore if both $\f,\fast$ are deep structured feature models (with distinct parameters in general), then~\Cref{ass:data+features} is satisfied whenever the feature maps $\f,\fast$ are centered\footnote{It is sufficent that e.g.\ $\phi_l$ is odd, and $\x_i$ is centered.} with respect to Gaussian covariates $\x_i$. As hinted in the introduction we will be particularly interested in one sub-class of~\Cref{deep feature model} known as \emph{Gaussian rainbow networks}.
\begin{definition}[Gaussian rainbow ensemble] 
\label{def:rainbow}
Borrowing the terminology of \cite{guth2023rainbow}, we define a fully-connected, $\depth$-layer \emph{Gaussian rainbow network} as a random variant of~\Cref{deep feature model} where for each $\ell$ the hidden-layer weights $\W_{\ell} = Z_{\ell}C_{\ell}^{1/2}$ are random matrices with $Z_{\ell}\in\R^{\p_{\ell+1}\times \p_{\ell}}$ having zero mean and i.i.d.\ variance $\sfrac{1}{p_\ell}$ Gaussian entries and $C_{\ell}\in\R^{p_{\ell}\times p_{\ell}}$ being uniformly bounded covariance matrices, which we allow to depend on previous layer weights $Z_1,\ldots,Z_{l-1}$. 
\end{definition}
Note that Gaussian rainbow networks above can be seen as a generalization of the deep random features model studied in \cite{schroder2023deterministic, bosch2023precise, Fan2020SpectraOT}, with the crucial difference that the weights are structured. 

\subsection*{Notation}
For square matrices $A\in\R^{n\times n}$ we denote the averaged trace by $\braket{A}:=n^{-1}\Tr A$, and for rectangular matrices $A\in\R^{n\times m}$ we denote the Frobenius norm by $\norm{A}_F^2:=\sum_{ij}\abs{a_{ij}}^2$, and the operator norm by $\norm{A}$. For families of non-negative random variables $X(n),Y(n)$ we say that $X$ is \emph{stochastically dominated} by $Y$, and write $X\prec Y$, if for all $\epsilon,D$ it holds that $P(X(n)\ge n^{\epsilon} Y(n))\le n^{-D}$ for $n$ sufficiently large. For a centered random vector \(x \in \R^d\) we denote its \emph{sub-Gaussian norm} as
        \(\norm{x}_{\psi_2} \coloneqq \inf_{\sigma 
        \geq 0} \{\E \exp^{\braket{v, x}} \leq \exp^{\frac{\norm{v}^2 \sigma^2}{2}}\ \forall\, v \in \R^d\}\).

%% file: sections/error.tex
Under~\Cref{ass:labels,ass:data+features} the generalization error from~\Cref{eq:gen err def} is given by
\begin{equation}\label{eq:gen_err}
    \begin{split}
        \mathcal E_\mathrm{gen}(\lambda)
        = \frac{\theta_\ast^\top\Psi\theta_\ast}{k} + \frac{\theta_\ast^\top Z X^\top G\Omega GX Z^\top \theta_\ast }{k p^2}  +  \frac{n}{p}\braket*{\frac{X^\top G\Omega GX\Sigma}{p}} -2 \frac{\theta_\ast^\top \Phi^\top G X Z^\top \theta_\ast}{kp},
    \end{split}
\end{equation}
in terms of the \emph{resolvent} $G=G(\lambda):=(XX^\top/p+\lambda)^{-1}$.

Our main result is a rigorous asymptotic expression for~\Cref{eq:gen_err}. To that end define, $m(\lambda)$ to be the unique solution to the equation
\begin{equation}\label{eq m def}
    \frac{1}{m(\lambda)}=\lambda +\braket*{\Omega \Bigl( I + \frac{n}{p} m(\lambda)\Omega\Bigr)^{-1}},
\end{equation}
and define
\begin{equation}\label{eq M def}
    M(\lambda)=\Bigl( \lambda + \frac{n}{p}\lambda m(\lambda)\Omega\Bigr)^{-1}
\end{equation}
which is the \emph{deterministic equivalent} of the resolvent, $M(\lambda)\approx G(\lambda)$, see~\Cref{thm MP} later. The fact that~\Cref{eq m def} admits a unique solution $m(\lambda)>0$ which is continuous in $\lambda$ follows directly from continuity and monotonicity. Moreover, from
\begin{equation*}
    0\le \braket*{\Omega \Bigl( I + \frac{n}{p} m\Omega\Bigr)^{-1}} \le \min\set*{\braket{\Omega}, \frac{\rank \Omega}{n}\frac{1}{m}}
\end{equation*}
we obtain the bounds
\begin{equation}
    \max\set*{\frac{1}{\lambda+\braket{\Omega}},\frac{1-\frac{\rank\Omega}{n}}{\lambda}} \le m(\lambda) \le \frac{1}{\lambda}.
\end{equation}
We also remark that $m(\lambda)$ depends on $\Omega$ only through its eigenvalues $\omega_1,\ldots,\omega_p$, while $M(\lambda)$ depends on the eigenvectors. The asymptotic expression~\Cref{eq Egen rmt} for the generalization error derived below depends on the eigenvalues of $\Omega$, the overlap of the eigenvectors of $\Omega$ with the eigenvectors of $\Phi$, and the overlap of the eigenvectors of $\Psi,\Phi$ with $\theta_\ast$.

\begin{theorem}\label{thm genRMT informal}
    Under~\Cref{ass:labels}, \Cref{ass:data+features} and \Cref{ass:dimensions} for fixed $\lambda>0$ we have the asymptotics
    \begin{equation}\label{E gen thm eq}
        \mathcal E_\mathrm{gen}(\lambda) = \mathcal E_\mathrm{gen}^\mathrm{rmt}(\lambda) + O_\prec\Bigl(\frac{1}{\sqrt{n}}\Bigr),
    \end{equation}
    in the proportional $n\sim k\sim p$ regime, where
    \begin{equation}\label{eq Egen rmt}
        \begin{split}
            \mathcal E_\mathrm{gen}^\mathrm{rmt}(\lambda)&:=\frac{1}{k}\theta_\ast^\top \frac{ \Psi-\frac{n}{p} m\lambda \Phi (M+\lambda M^2)\Phi^\top}{1-\frac{n}{p}(\lambda m)^2\braket{\Omega M\Omega M}}\theta_\ast  + \braket{\Sigma}  \frac{ (\lambda m)^2\frac{n}{p} \braket{ M\Omega M\Omega  }}{1-\frac{n}{p}(\lambda m)^2\braket{\Omega M\Omega M}}.
        \end{split}
    \end{equation}
    In the general case of comparable parameters we have the asymptotics with a worse error of\footnote{This allows to identify the leading order of the generalization error as long as the ratio of the largest and smallest parameter is much smaller than the square-root of the smallest one.}
    \[
        \frac{1}{\sqrt{\min\set{n,p,k}}}\Bigl(1 + \frac{\max\set{n,p,k}}{\min\set{n,p,k}}\Bigr).
    \]
\end{theorem}
\begin{remark}[Relation to previous results]
    We focus on the misspecified case as this presents the main novelty of the present work. In the wellspecified case $Z=X$ our model essentially reduces to linear regression with data distribution $x=\f(\x)$. There has been extensive research on the generalization error of linear regression, see e.g.\ in~\cite{2303.01372,Dobriban2015HighDimensionalAO,2103.09177,2210.08571} and the references therein. 
    \begin{enumerate}[label=(\alph*)]
        \item We confirm Conjecture 1 of~\cite{Loureiro2021CapturingTL} under~\cref{ass:data+features}. The expression for the error term in~\Cref{thm genRMT informal} matches the expression obtained in~\cite{Loureiro2021CapturingTL} for a Gaussian covariates teacher-student model. 
        \item Independently and concurrently to the current work~\cite{2312.09194} (partially confirming a conjecture made in~\cite{louart2018random}) obtained similar results under different assumptions. Most importantly~\cite{2312.09194} considers one-layer unstructured random feature models and computes the \emph{empirical generalization error} for a deterministic data set, while we consider general Lipschitz features of random data, and compute the generalization error. 
        \item In the unstructured random feature model~\cite{Mei2021GeneralizationEO,2008.06786} obtained an expression for the generalization error under the assumption that the target model is linear or rotationally invariant. 
    \end{enumerate}
\end{remark}
The novelty of~\Cref{thm genRMT informal} compared to many of the previous works is, besides the level of generality, two-fold:
\begin{enumerate}[label=(\roman*)]
    \item\label{gen error det equ} We obtain a deterministic equivalent for the generalization error involving the population covariance $\Phi$ and the sample covariance $XZ^\top$ in the general misspecified setting.
    \item\label{anisotropic det equ} Our deterministic equivalent is \emph{anisotropic}, allowing to evaluate~\Cref{eq:gen_err} for \emph{fixed} targets $\theta_\ast$ and structured noise covariance $\Sigma\ne I$.
\end{enumerate}
Some of the previous rigorous results on the generalization error of ridge regression have been limited to the well-specified case, $X=Z$, since in this particular case the second term of~\Cref{eq:gen_err} can be simplified to 
\begin{equation}
    \frac{XX^\top}{p} G\Omega G \frac{XX^\top}{p} = (1-\lambda G)\Omega(1-\lambda G).
\end{equation}
When computing deterministic equivalents for terms as $G\Omega G$, some previous results have relied on the ``trick'' of differentiating a generalized resolvent matrix $\wt G(\lambda,\lambda'):=(XX^\top/p + \lambda'\Omega + \lambda)^{-1}$ with respect to $\lambda'$. Our approach is more robust and not limited to expressions which can be written as certain derivatives. 

To illustrate~\Cref{anisotropic det equ}, the conventional approach in the literature to approximating e.g.\ the third term on the right hand side of~\Cref{eq:gen_err} in the case $\Sigma=I$ would be to use the cyclicity of the trace to obtain
\begin{equation}\label{old approach}
    \begin{split}
        \frac{1}{p^2}\Tr X^\top G\Omega G X &= \frac{1}{p}\Tr G \frac{X X^\top}{p} G\Omega \\
        &= \braket{ G\Omega} -\lambda \braket{G^2\Omega}.
    \end{split}
\end{equation}
Then upon using~\Cref{eq m def} and $\braket{G\Omega}\approx\braket{M\Omega}$, the first term of~\Cref{old approach} can be approximated by $1/(\lambda m(\lambda))-1$, while for the second term it can be argued that this approximation  also holds in derivative sense to obtain
\[\braket{G^2 \Omega}=-\frac{\dif}{\dif\lambda}\braket{G\Omega} \approx -\frac{\dif}{\dif\lambda} \frac{1}{\lambda m(\lambda)} = \frac{\lambda m'(\lambda)+m(\lambda)}{(\lambda m(\lambda))^2}
\]
By differentiating~\Cref{eq m def}, solving for $m'$ and simplifying, it can be checked that this result agrees with the second term of~\Cref{eq Egen rmt} in the special case $\Sigma=I$. However, it is clear that any approach which only relies on \emph{scalar} deterministic equivalents is inherently limited in the type of expressions which can be evaluated. Instead, our approach involving \emph{anisotropic deterministic equivalents} has no inherent limitation on the structure of the expressions to be evaluated.

An alternative to evaluating rational expressions of $X,Z$, commonly used in similar contexts, is the technique of \emph{linear pencils}~\cite{2008.06786, 2312.09194}. The idea here is to represent rational functions of $X,Z$ as blocks of inverses of larger random matrices which depend linearly $X,Z$. The downside of linear pencils is that even for simple rational expressions the linearizations become complicated, sometimes even requiring the use of computer algebra software for the analysis\footnote{For instance~\cite{2008.06786} used block matrices with up to $16\times 16$ blocks in order to evaluate the asymptotic test error.} In comparison we believe that our approach is more direct and flexible.

\subsection{Proof of~\Cref{thm genRMT informal}}
We present the proof of~\Cref{thm genRMT informal} in details in~\Cref{appendix gen error}. The main steps and ingredients for the proof of~\Cref{thm genRMT informal} consist of the following:
\begin{description}
    \item[Concentration:] As a first step we establish \emph{concentration estimates} for Lipschitz functions of $X,Z$ and its columns. A key aspect is the concentration of quadratic forms in the columns $x_i:=\f(\x_i)$ of $X$:
        \begin{equation*}
            \abs{x_i^\top A x_i - \E x_i^\top A x_i}= \abs{x_i^\top A x_i-\Tr \Omega A}\prec \norm{A}_F
        \end{equation*}
        which follows from the Hanson-Wright inequality~\cite{1409.8457}. The concentration step is very similar to analagous considerations in previous works~\cite{chouard2022quantitative,1702.05419} but we present it for completeness. The main property used extensively in the subsequent analysis is that traces of resolvents with deterministic observables concentrate as
        \begin{equation}\label{concentration AG}
            \abs{\braket{A[G(\lambda)-\E G(\lambda)]}} \prec \frac{\braket{\abs{A}^2}^{1/2}}{n\lambda^{3/2}}.
        \end{equation}
    \item[Anisotropic Marchenko-Pastur Law:] 
    As a second step we prove an anisotropic Marchenko-Pastur law for the resolvent $G$, of the form:
        \begin{theorem}\label{thm MP}
            For arbitrary deterministic matrices $A$ we have the high-probability bound
            \begin{equation}\label{mp main eq}
                \abs{\braket{(G(\lambda)-M(\lambda)A}} \prec \frac{\braket{\abs{A}^2}}{n\lambda^3},
            \end{equation}
            in the proportional $n\sim p$ regime\footnote{See the precise statement in the comparable regime in~\Cref{eq MP comp} later}.
        \end{theorem}
        \begin{remark}
            Tracial Marchenko-Pastur laws (case $A=I$ above) have a long history, going back to~\cite{marcenkopastur} in the isotropic case $\Omega=I$,~\cite{SILVERSTEIN1995331} in the general case with separable covariance $x=\sqrt{\Omega}z$ and~\cite{baizhou} under quadratic form concentration assumption. Anisotropic Marchenko-Pastur laws under various conditions and with varying precision have been proven e.g.\ in~\cite{RUBIO2011592,chouard2022quantitative,louart2018random,10.1007/s00440-016-0730-4}.
        \end{remark}        
        For the proof of~\Cref{thm MP} the resolvent $\wc G:=(X^\top X/p+\lambda)^{-1}\in\R^{n\times n}$ of the \emph{Gram matrix} $X^\top X$ plays a key role. The main tool used in this step are the commonly used \emph{leave-one-out identities}, e.g.
        \begin{equation}\label{Gx loo}
            Gx_i=\lambda\wc G_{ii} G_{-i}x_i, \quad G_{-i}:=\Bigl(\sum_{j\ne i}\frac{x_jx_j^\top}{p} + \lambda \Bigr)^{-1}
        \end{equation}
        which allow to decouple the randomness due the $i$-th column from the remaining randomness. 
        Such identities are used repeatedly to derive the approximation
        \begin{equation}\label{EG approx MG}
            \E G \approx \Bigl(\frac{n}{p}\lambda \braket{\E \wc G}\Omega + \lambda\Bigr)^{-1}
        \end{equation}
        in Frobenius norm, which, together with the relation $1-\lambda\braket{\wc G}=\frac{p}{n}\bigl(1-\lambda\braket{G}\bigr)$ between the traces of $G$ and $\wc G$, yields a self-consistent equation
        for $\braket{\wc G}$. This self-consistent equation is an approximate version of~\Cref{eq m def}, justifying the definition of $m$. The \emph{stability} of the self-consistent equation then implies the averaged asymptotic equivalent
        \begin{equation}
            \abs{m-\braket{\E\wc G}}\lesssim\frac{1}{n\lambda^2}.
        \end{equation}
        and therefore by~\Cref{EG approx MG} finally
        \begin{equation}
            \norm{M-\E G}_F \lesssim \frac{1}{n^{1/2}\lambda^3},
        \end{equation}
        which together with~\Cref{concentration AG} implies~\Cref{thm MP}.

        Compared to most previous anisotropic deterministic equivalents as in~\cite{10.1007/s00440-016-0730-4} we measure the error of the approximation~\Cref{mp main eq} with respect to the Frobenius norm of the observable $A$. As in the case of unified local laws for Wigner matrices~\cite{2203.01861} this idea renders the separate handling of quadratic form bound unnecessary, considerably streamlining the proof. To illustrate the difference note that specializing $A$ to be rank-one $A=xy^\top$ in
        \[
            \begin{split}
                \abs{y^\top(G-M)x}=\abs{\Tr(G-M)A}&\prec \begin{cases}\norm{A} \\ \braket{\abs{A}^2}^{1/2}
                \end{cases}
            \end{split}
        \]
        results in a trivial estimate $\norm{x}\norm{y}$ in the case of the spectral norm, and in the optimal estimate $\norm{x}\norm{y}/\sqrt{p}$ in the case of the Frobenius norm.

    \item[Anisotropic Multi-Resolvent Equivalents: ] The main novelty of the current work lies in~\Cref{prop multi res} which asymptotically evaluates the expressions on the right-hand-side of~\Cref{eq:gen_err}. A key property of the deterministic equivalents is that the approximation is \emph{not} invariant under multiplication. E.g.\ for the last term in~\Cref{eq:gen_err} we have the approximations $G\approx M$ and $\frac{1}{n}XZ^\top=\frac{1}{n}\sum x_i z_i^\top \approx \Phi$, while for the product the correct deterministic equivalent is
        \begin{equation}
            G \frac{X Z^\top}{n} \approx \lambda m M\Phi,
        \end{equation}
        i.e.\ the is an additional factor of $m\lambda$. In this case the additional factor can be obtained from a direct application of the leave-one-out identity~\Cref{Gx loo} to the product $G\frac{XZ^\top}{n}$, but the derivation of the multi-resolvent equivalents requires more involved arguments. When expanding the multi-resolvent expression $\braket{GAGB}$ we obtain an approximative self-consistent equation of the form 
\begin{equation*}
   \braket{GAGB}\approx \braket{MAMB} + \frac{n}{p}(m\lambda)^2 \braket{M B M \Omega} \braket{G A G\Omega}.
\end{equation*}
Using a stability analysis this yields a deterministic equivalent for the special form $\braket{GAG\Omega}$ which then can be used for the general case. The second term of~\Cref{eq:gen_err} requires the most careful analysis due to the interplay of the multi-resolvent expression and the dependency among $Z,X$. 
\end{description}

%% file: sections/linearization.tex
Theorem \ref{thm genRMT informal} characterizes the test error for learning using Lipschitz feature maps as a function of the three features population (cross-)covariances $\Omega, \Phi, \Psi$. For the particular case where both the target and learner feature maps are drawn from the Gaussian rainbow ensemble from~\Cref{def:rainbow}, these population covariances can be expressed in closed-form in terms of combinations of products of the weights matrices. Consider two rainbow networks
\begin{equation}
    \begin{split}
    \f(\x)&=\f_L(W_L\f_{L-1}(\dots \f_1(W_1 \x)))\\
     \fast(\x)&=\wt{\f}_{\wt{L}}(V_{\wt{L}} \wt{\f}_{\wt{L}-1}(\dots \wt{\f}_1(V_1 \x)))
     \end{split}
\end{equation}
with depths $L,\wt{L}$. The approach we introduce here is in theory capable of obtaining linear or polynomial approximations to $\Omega,\Phi,\Psi$ under very general assumptions. However, for definiteness we focus on a class of correlated rainbow networks in which we allow the $k$-th row of $W_\ell$ to be correlated only to the $k$-th row of $W_{\ell'},V_{\ell'}$ as this allows for particularly simple expressions for the linearized covariances\footnote{The identity matrices in~\Cref{eq lin cov} are a direct consequence of this assumption. In case of weight matrices with varying row-norms or covariances across rows the resulting expression would be considerably more complicated.}. Note that we explicitly allow for weights to be correlated across layers. 
\begin{assumption}[Correlated rainbow networks]\label{corr rainbow}
    By symmetry we assume without loss of generality $L\le \wt L$. Furthermore, for all \(\ell \leq L \leq \wt L\), we assume
    \begin{enumerate}[label=(\alph*)]
        \item All the internal widths $p_\ell$ of $W_\ell,V_\ell$ agree,
        \item The rows $w_\ell,v_\ell$ of $W_\ell,V_\ell$ are i.i.d. with mean zero and 
        \begin{equation*}
        C_\ell := p_\ell \E w_\ell w_\ell^\top, \; \wt C_\ell:=p_\ell\E v_\ell v_\ell^\top,\; \wc C_\ell:=p_\ell\E w_\ell v_\ell^\top,
        \end{equation*}
        with \(\norm{C_\ell} + \norm{\wt C_\ell} + \norm{\wc C_\ell}\lesssim 1\),
        \item Asymptotic orthogonality of the rows of \(W_\ell, V_\ell\). Let \(w, w'\) be two independent copies of a row of \(W_\ell\). Then, \(\braket{w, w'} \prec d^{-1/2}\), same for \(V_\ell\),
        \item The rows of \(W_\ell, V_\ell\) are sub-Gaussian random vectors: 
        \begin{equation}
            \norm{w_\ell}_{\psi_2} + \norm{v_\ell}_{\psi_2} = O(d^{-1/2}),
        \end{equation}
        \item Centered activation functions \(\f_\ell, \wt \f_\ell: \E_{\x} \f_{\ell}(W_{\ell} \f_{\ell - 1}(\ldots \f_1(W_1 \x))) = 0\), same for \(\wt \f_\ell\) (see~\Cref{ass:data+features}).
    \end{enumerate}
\end{assumption}
Under~\Cref{corr rainbow} the \emph{linearized population covariances} can be defined recursively as follows:
\begin{definition}[Linearized population covariances]
    \label{def: linearized_covs}
    Define the sequence of matrices $\Omega_\ell^\mathrm{lin},\Phi_\ell^\mathrm{lin}, \Psi_\ell^\mathrm{lin}$ by the recursions
    \begin{equation}\label{eq lin cov}
    \begin{split}
         & \Omega_{\ell}^{\mathrm{lin}}=(\kappa^1_\ell)^2 W_\ell\Omega_{\ell-1}^{\mathrm{lin}}W_\ell^\top+(\kappa^{*}_\ell)^2\mathbb{I}_{p_\ell}                             \\
         & \Psi_{\ell}^{\mathrm{lin}}=(\tilde{\kappa}^1_\ell )^2V_\ell\Psi_{\ell-1}^{\mathrm{lin}}V_\ell^\top+(\wt{\kappa}^{*}_\ell)^2\mathbb{I}_{p_\ell} \\
         & \Phi_{\ell}^{\mathrm{lin}}=\kappa^1_\ell\tilde{\kappa}^1_\ell W_\ell\Phi_{\ell-1}^{\mathrm{lin}} V_\ell^\top +(\wc{\kappa}^*_\ell)^2\mathbb{I}_{p_\ell},
         \end{split}
    \end{equation}
    with $\Omega_0^{\mathrm{lin}}=\Psi_0^{\mathrm{lin}}=\Phi_0^{\mathrm{lin}}=\Omega_0$ the input covariance.
    The coefficients $\{\kappa_\ell^1,\tilde{\kappa}_\ell^1,\kappa_\ell^*,\tilde{\kappa}_\ell^*, \wc{\kappa}^*_\ell\}$ are defined by the recursion
    \begin{equation}
        \kappa_\ell^1 := \E \f_\ell'(N_\ell), \quad \wt \kappa_\ell^1:=\E \wt\f_\ell'(\wt N_\ell)
    \end{equation}
    and 
    \begin{equation}
        \begin{split}
            \kappa^\ast_\ell &= \sqrt{\E[\f_\ell(N_\ell)^2]-r_\ell(\kappa^1_\ell)^2}                                         \\
            \tilde{\kappa}^*_\ell&= \sqrt{\E [\wt{\f}_\ell(\wt N_\ell)^2]-\tilde{r}_\ell(\tilde{\kappa}^1_\ell)^2} \\
            \wc{\kappa}^*_\ell&=\sqrt{\E [\f_\ell(N_\ell )\wt{\f}_\ell(\wt N_\ell)]-\check{r}_\ell \kappa^1_\ell\wt{\kappa}^1_\ell},
        \end{split}
    \end{equation}
    where $N_\ell,\wt N_\ell$ are jointly mean-zero Gaussian with $\E N_\ell^2=r_\ell$, $\E \wt N_\ell^2=\wt r_\ell$, $\E N_\ell\wt N_\ell=\wc r_\ell$, with 
    \begin{equation*}
        r_{\ell}=\Tr[C_\ell \Omega_{\ell-1}^{\mathrm{lin}}], \; \wt{r}_{\ell}=\Tr[\wt{C}_\ell \Psi_{\ell-1}^{\mathrm{lin}}],\; \wc{r}_{\ell}=\Tr[\wc{C}_\ell^\top \Phi_{\ell-1}^{\mathrm{lin}}].
    \end{equation*}
    Finally, for $\tilde{L}\ge \ell\ge L+1$, define
    \begin{equation}
    \begin{split}
         \Phi_{\ell}^{\mathrm{lin}}&=\wt{\kappa}^1_\ell \Phi_{\ell-1}^{\mathrm{lin}}V_\ell^\top,
         \end{split}
    \end{equation}
    with still $\wt\kappa_\ell^1,\wt\kappa_\ell^\ast$ just as before, and $ \Psi_{\ell}^{\mathrm{lin}}$ with the same recursion \eqref{eq lin cov}.
\end{definition}

\begin{Conjecture}
    The populations covariances $\Omega, \Phi,\Psi$ involved in Theorem \ref{thm genRMT informal} can be asymptotically approximated with the last iterates of the linear recursions of Definition \ref{def: linearized_covs}, i.e.
    \begin{align}
        \norm{\Omega-\Omega_L^\mathrm{lin}}_F + \norm{\Psi-\Psi_{\tilde{L}}^\mathrm{lin}}_F + \norm{\Phi-\Phi_{\tilde{L}}^\mathrm{lin}}_F &\lesssim 1
    \end{align}
\end{Conjecture}

Note that the linearization from~\Cref{def: linearized_covs} also provides good approximation to the population covariances $\Omega_\ell,\Phi_\ell,\Psi_\ell$ of the post-activations at intermediate layers $\ell$. The method we use to rigorously derive the linearizations is in theory applicable to any depths, however the estimates quickly become tedious. To keep the present work at a manageable length we provide a rigorous proof of concept only for the simplest multi-layer case.
\begin{theorem}\label{theo lin}
    Under~\Cref{corr rainbow} with $L=\wt L = 2$, we have 
    \begin{equation*}
        \begin{split}
        \norm{\Omega_1-\Omega_1^\mathrm{lin}}_F + \norm{\Psi_1-\Psi_1^\mathrm{lin}}_F + \norm{\Phi_1-\Phi_1^\mathrm{lin}}_F &\prec 1,\\
        \norm{\Omega_2 - \Omega_2^{\mathrm{lin}}}_F + 
        \norm{\Psi_2-\Psi_2^\mathrm{lin}}_F + \norm{\Phi_2-\Phi_2^\mathrm{lin}}_F&\prec 1.
        \end{split}
    \end{equation*} 
\end{theorem}
\begin{remark}[Comparison]
    The approach we take here is somewhat different from previous works~\cite{schroder2023deterministic,Fan2020SpectraOT,2306.05850} on (multi-layer) random feature models. In these previous results, the deterministic equivalent for the resolvent was obtained using primarily the randomness of the weights, resulting in relatively stringent assumptions (Gaussianity and independence between layers). This layer-by-layer recursive approach resulted in a deterministic equivalent for the resolvent which is \emph{consistent} with a sample covariance matrix with linearized population covariance. Here we take the direct approach of considering feature models with arbitrary structured features, and then linearize the population covariances in a separate step for random features.  
\end{remark}

\subsection{Proof of~\Cref{theo lin}}
    We sketch the main tools used in the argument and we refer the reader to~\cref{prop:lin_1_layer} and~\cref{thm:lin_app} for the formal proof.
    In the proof, we crucially rely on the theory of Wiener chaos expansion and Stein's method (see~\cite{nourdin2012normal}). Gaussian Wiener chaos is a generalization of Hermite polynomial expansions, which previously have been used for approximate linearization \cite{Fan2020SpectraOT,schroder2023deterministic} in similar contexts. The basic idea is to decompose random variables $F=F(\x)$ which are functions of the Gaussian random vector $\x$, into pairwise uncorrelated components 
    \begin{equation}
        F = \E F + \sum_{p\ge 1} I_p\Bigl(\frac{\E D^p F}{p!}\Bigr), 
    \end{equation}
    where $I_p$ is a so called \emph{multiple integral} (generalizing Hermite polynomials) and $D^p$ is the $p$-th Malliavin derivative. By applying this to the one-layer quantities $\f_1(w^\top \x),\f_1(u^\top \x)$ we obtain, for instance 
    \begin{equation}
    \begin{aligned}
        &\E \f_1(w^{\top} \x) \f_1(v^{\top} \x) \\
        &\quad = \sum_{p \geq 1} \frac{1}{p!} \E \f_1^{(p)}(w^\top \x) \E \f_1^{(p)}(u^\top \x) \braket{w, v}^p,
    \end{aligned}
    \end{equation}
    which for independent $w,v$ we can truncate after $p=1$, giving rise to the linearization. 
    
    For the multi-layer case we combine the chaos expansion with Stein's method in order to prove \emph{quantitative central limit theorems} of the type 
    \begin{equation}
        d_W(F,N) \lesssim \E \abs{\E F^2 - \braket{DF,-DL^{-1}F}}
    \end{equation}
    for the Wasserstein distance $d_W$, where
    \begin{equation}
        F := w^\top \phi_1(W\x),\quad  N\sim \mathcal N(0,\E F^2),
    \end{equation}
    and $L^{-1}$ is the pseudo-inverse of the \emph{generator of the Ornstein–Uhlenbeck semigroup}.

\begin{figure}[t]
    \centering
    \includegraphics{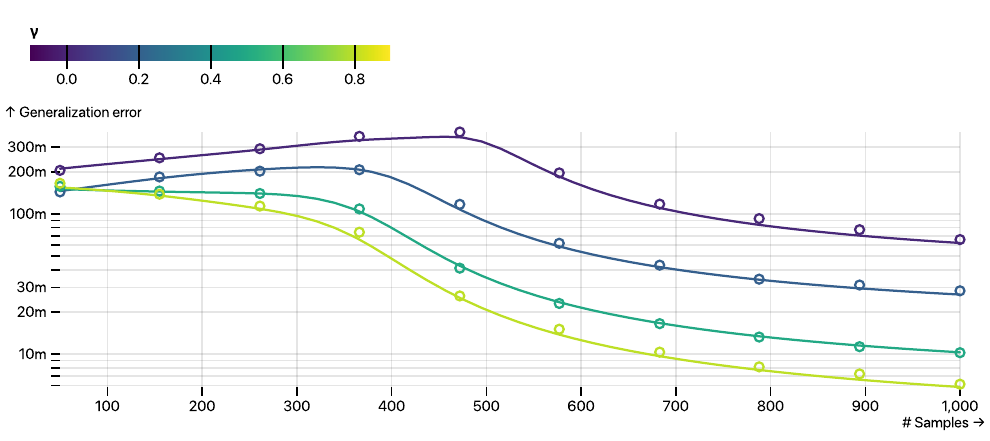}
    \caption{Test error for a target $\theta_*^\top \tanh(W_* x)$, when learning with a four-layer Gaussian rainbow network with feature map $\f(x)=\tanh(W_3\tanh(W_2\tanh(W_1x)))$. All width were taken equal to the input dimension $d$, and the regularization employed is $\lambda=10^{-4}$. The student weights are correlated across layers, with $W_1=W_2$, and the covariance $C_3$ of $W_3$ depending on $W_1$ as $C_3=(W_1W_1^\top+\sfrac{1}{2}\mathbb{I}_d)^{-1}$. Target/student correlations are also present, with $\check{C}_1=\sfrac{1}{2}\mathbb{I}_d$. The covariances $C_1,C_2,\tilde{C}_1$ were finally taken to have a spectrum with power-law decay, parametrized by $\gamma$. All details are provided in App. \ref{app: numerics}. Solid lines: theoretical prediction of Theorem \ref{thm genRMT informal}, in conjunction with the closed-form expression for the features population covariance of Definition \ref{def: linearized_covs}. Circles : numerical simulations in $d=1000$. 
    }
    \label{fig:PL}
\end{figure}

\subsection{Discussion of Theorem \ref{theo lin}}

The population covariances thus admit simple approximate closed-form expressions as linear combinations of products of relevant weight matrices. These expressions generalize similar linearizations introduced in \cite{Cui2023,schroder2023deterministic, bosch2023precise, Fan2020SpectraOT,2306.05850} for the case of weights which are both unstructured and independent, and iteratively build upon earlier results for the two-layer case developed in \cite{Mei2019TheGE, Gerace2020GeneralisationEI, Goldt2021TheGE, Hu2020UniversalityLF}.  In fact, the expressions leveraged in these works can be recovered as a special case for $C_\ell=\tilde{C}_\ell=\mathbb{I}_{p_\ell}$ (isotropic weights) and $\check{C}_\ell=0$ (independence). Importantly, note that possible correlation between weights across different layers do not enter in the reported expressions. In practice, we have observed in all probed settings the test error predicted by Theorem \ref{thm genRMT informal}, in conjunction with the linearization formulae for the features covariance, to match well numerical experiments. 

Figure \,\ref{fig:PL} illustrates a setting where many types of weights correlations are present. It represents the learning curves of a four-layer Gaussian rainbow network with feature map $\tanh(W_3\tanh(W_2\tanh(W_1\x)))$, learning from a two-layer target $\theta_*^\top \tanh(V\x)$. To illustrate our result, we consider both target/student correlations $\wc{C}_1=\sfrac{1}{2}\mathbb{I}_d$, and inter-layer correlations $W_1=W_2$. We furthermore took the covariance of the third layer to depend on the weights of the first layer, $C_3=(W_1W_1^\top+\sfrac{1}{2}\mathbb{I}_d)^{-1}$. In order to have structured weights, the covariances $\wt{C}_1, C_1,C_2$ were chosen to have a power-law spectrum. All details on the experimental details and parameters are exhaustively provided in Appendix \ref{app: numerics}. Note that despite the presence of such non-trivial correlations, the theoretical prediction of Theorem \ref{thm genRMT informal} using the linearized closed-form formulae of Def.~\ref{def: linearized_covs} for the features covariances (solid lines) captures compellingly the test error evaluated in numerical experiment (crosses).

Finally, we note that akin to \cite{schroder2023deterministic}, as a consequence of the simple linear recursions, it follows that the Gaussian rainbow network feature map $\f$ shares the same second moments, and thus by Theorem \ref{thm genRMT informal} the same test error, as an equivalent \textit{linear stochastic} network $\f^{\mathrm{lin}}=\psi_L \circ\dots\circ\psi_1$, with
\begin{align}
\label{eq:linearized-layer}
    \psi_\ell(x)=\kappa^1_\ell W_\ell x+\kappa^*_\ell\xi_\ell
\end{align}
where $\xi_\ell\sim\mathcal{N}(0,\mathbb{I}_{p_\ell})$ a stochastic noise. This equivalent viewpoint has proven fruitful in yielding insights on the implicit bias of RFs \cite{schroder2023deterministic, Jacot2020} and on the fundamental limitations of deep networks in the proportional regime \cite{Cui2023}. In the~\Cref{section lin grad} we push this perspective further, by heuristically finding that the linearization and Theorem \ref{thm genRMT informal} can also describe deterministic networks trained with gradient descent in the lazy regime.

%% file: sections/trained.tex
\begin{figure}[t]
    \centering
    \includegraphics{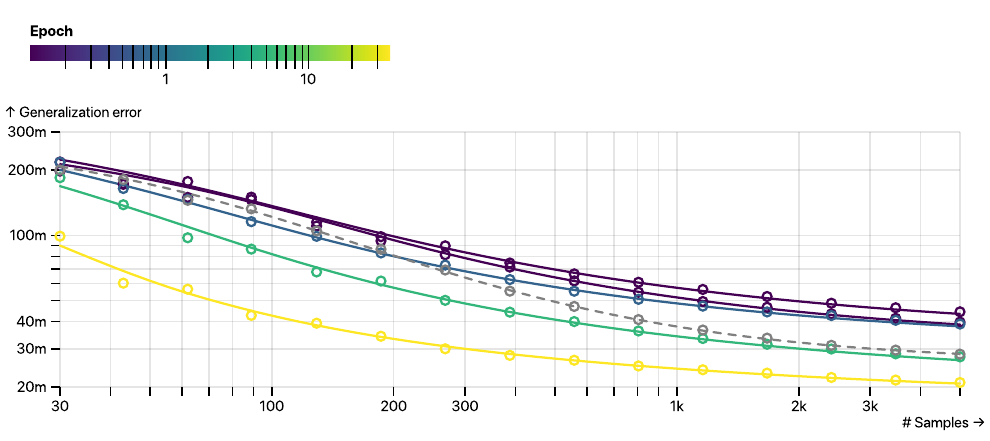}
    \caption{
    Test error when training the readout layer only of a relu-activated three-layer neural network during training, using the \texttt{Tensorflow} implementation of the Adam \cite{kingma2014adam} optimizer, over $120$ epochs with batch size $128$. (dashed): ridge regression. The data is sampled from a Gaussian distribution with mean and variance matching the distribution of MNIST images. In all training procedures, the regularization parameter has been numerically optimized. Solid lines represent the theoretical prediction of Theorem \ref{thm genRMT informal}, dots represent numerical experiments. For more details we refer to~\Cref{synth MNIST}. 
    }
    \label{fig:inductive_bias}
\end{figure}

The previous discussion addressed feature maps associated to random Gaussian networks. However, note that the linearization itself only involves products of the weights matrices, and coefficient depending on weight covariances which can straightforwardly be estimated therefrom. The linearization \ref{def: linearized_covs} can thus be readily heuristically evaluated  for feature maps associated to deterministic \textit{trained} finite-width neural networks. As we discuss later in this section, the resulting prediction for the test error captures well the learning curves when re-training the readout weights of the network in a number settings. Naturally, such settings correspond to lazy learning regimes \cite{Jacot2020}, where the network feature map is effectively \textit{linear}, thus little expressive. However, these trained feature map, albeit linear, can still encode some inductive bias, as shown by \cite{Ba2022} for one gradient step in the shallow case. In this section, we briefly explore these questions for fully trained deep networks, through the lens of our theoretical results.

Fig.\,\ref{fig:inductive_bias} contrasts the test error achieved by linear regression (red), and regression on the feature map associated to a three-layer student at initialization (green) and after $3000$ epochs of end-to-end training using full-batch Adam \cite{kingma2014adam} at learning rate $10^{-4}$ and weight decay $10^{-3}$ over $n_0=1400$ training samples (blue). For all curves, the readout weights were trained using ridge regression, with regularization strength optimized over using cross-validation. Solid curves indicate the theoretical predictions of Thm. \ref{thm genRMT informal} leveraging the closed-form linearized formulae \ref{def: linearized_covs} for the features covariance. Interestingly, even for the deterministic trained network features, the formula captures the learning curve well. This observation temptingly suggests to interpret the feature map $\f(x)$ as the stochastic linear map
\begin{align}
 \label{eq: effective_linear_net}
 \f^g(x)=W_{\mathrm{eff.}}x+C_{\mathrm{eff.}}^{\sfrac{1}{2}}\xi
\end{align}
where $W_{\mathrm{eff.}}\in\R^{p\times d}$ is proportional to the product of all the weight matrices
\begin{align}
    W_{\mathrm{eff.}}=\left(\prod\limits_{\ell=1}^L \kappa^1_\ell\right) \hat{W}_L\hat{W}_{L-1}\dots \hat{W}_1 
\end{align}
and $\xi\sim\mathcal{N}(0,\mathbb{I}_p)$ is a stochastic noise colored by the covariance
\begin{align}
    \label{eq:effective_cov}
    C_{\mathrm{eff.}}\equiv & \sum\limits_{\ell=1}^{L-1}
    \left(\kappa^*_\ell\prod\limits_{s=\ell+1}^L
    \kappa_s^1
    \right)^2  \hat{W}_L\dots \hat{W}_{\ell+1}\hat{W}_{\ell+1}^\top\dots \hat{W}_L^\top\notag \\
                    & +  (\kappa^*_L)^2\mathbb{I}_{p_L}.
\end{align}
Note that the effective linear network \eqref{eq: effective_linear_net} simply corresponds to the composition of the equivalent stochastic linear layers \eqref{eq:linearized-layer}. A very similar expression for the covariance of the effective structured noise \eqref{eq:effective_cov} appeared in \cite{schroder2023deterministic} for the random case with unstructured and untrained random weights. The effective linear model \eqref{eq: effective_linear_net} affords a concise viewpoint on a deep finite-width non-linear network trained in the lazy regime. On an intuitive level, during training, the network effectively tunes the two matrices $W_{\mathrm{eff.}},C_{\mathrm{eff.}}$ which parametrize the effective model \eqref{eq: effective_linear_net}. Indeed, the interplay between these two matrices -- both depending on the weights $\hat{W}_{\ell}$ -- defines the inductive bias of the trained network in the high-dimensional regime, which ultimately determines the generalization properties of the network. To see this explicitly, consider the ridge regression problem on the effective linear features in \cref{eq: effective_linear_net}. Changing variables $\beta = \sfrac{C_{\mathrm{eff.}}^{\sfrac{1}{2}}\theta}{\sqrt{p}}$ and assuming for simplicity that $C_{\mathrm{eff.}}$ is invertible, this yields the following effective problem:\looseness=-1
\begin{align}
    \label{eq:def:erm_equiv}
    \underset{\beta\in\R^{\p}}{\min}~&\sum\limits_{i\in[\n]}\left(\y_{i}-\beta^\top(C_{\mathrm{eff.}}^{-\sfrac{1}{2}}W_{\mathrm{eff.}}\x_{i}+\xi_i)\right)^{2}\!\! +p\lambda\beta^{\top}C^{-1}_{\mathrm{eff.}}\beta.\notag
\end{align}
In this basis, the effective linear features has two components: an irreducible isotropic noise $\xi_{i}$ and a term $C^{-\sfrac{1}{2}}_{\mathrm{eff.}}W_{\mathrm{eff.}}\x_{i}$ controlling the (linear) representation of the training data. A key difference with respect to the deep unstructured case of \cite{schroder2023deterministic} is that the effective $\ell_2$-regularization is anisotropic.

For a (typically employed) random isotropic initialization, the initial network is equivalent to a unstructured dRF. In particular, the unstructured dRF inductive bias \cite{Jacot2020} is not aligned to the target, treating all directions equally. In terms of generalization, since the effective linear features are noisy, this implies that the initial generalization error is lower-bounded by the best ridge estimator \cite{schroder2023deterministic}. As the network is trained, the weights $\hat{W}_{\ell}$ adapt to the target, implying that even in the regime where the linearization in \cref{eq: effective_linear_net} holds, the effective linear problem \cref{eq: effective_linear_net} can regularize different directions adaptively, potentially outperforming the ridge regression baseline. The fact that the optimal regularization for ridge regression on linear feature maps might be anisotropic has been explored in detail in \cite{Wu2020OnTO}.\looseness=-1

The learning of beneficial inductive biases over training is illustrated by Fig.\,\ref{fig:inductive_bias} for synthetic data. Despite the fact that all represented feature maps are effectively just linear feature maps, they can still encode very different biases, yielding different phenomenology. In particular, remark that, when trained over a sufficient number of epochs, the trained feature map outperforms by ridge regression on the whole range of probed sample complexities -- suggesting the trained weights $W_{\mathrm{eff.}},C_{\mathrm{eff.}}$ learned some form of helpful inductive bias, and allow for a more performant linear model. A similar qualitative behaviour can also be observed in real data sets, as illustrated in Fig.\ref{fig:real_MNIST}.

\begin{figure}
    \centering
    \includegraphics{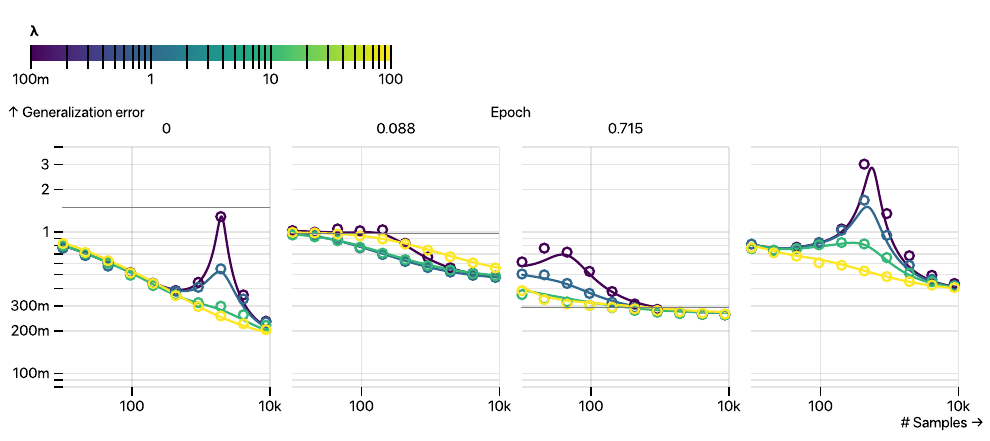}
    \caption{Test error when re-training the readout layer only of an Adam-optimized relu-activated three-layer neural network, trained on a regression task on MNIST. Labels are $+1$ (resp. $-1$) for even (resp. odd) digits. Solid lines represent the theoretical prediction of Theorem \ref{thm genRMT informal}, dots represent numerical experiments on the real dataset. Different colors indicate different reguarization strengths $\lambda$. Different panels correspond to different training times. All details are provided in App.~\ref{app:real_data}.}
    \label{fig:real_MNIST}
\end{figure}

\section{Concluding remarks}
\paragraph{Real data ---}We observe that the theoretical predictions of Theorem \ref{thm genRMT informal} also capture the learning curves of trained networks on some \textit{real} datasets, when retraining the readout only using ridge regression, provided the features covariances $\Omega,\Phi,\Psi$ are estimated from data. Fig.\,\ref{fig:real_MNIST} contrasts the theoretical characterization of Theorem \ref{thm genRMT informal} with numerical experiments on MNIST \cite{lecun1998gradient}, for a three-neural network optimized with Adam \cite{kingma2014adam}, revealing overall good agreement. All experimental details are reported in Appendix \ref{app:real_data}. Note that closely related observations have also been made in \cite{Loureiro2021CapturingTL}.

\paragraph{Limitations ---} Our results provides an insight in the inductive bias of trained deep rainbow networks. However, as discussed in \cite{guth2023rainbow}, this only captures a subset of neural networks. Understanding the boundaries of applicability of the Gaussian rainbow framework (and hence of our theory) is an interesting problem. A recent line of work investigating the properties of two-layer neural networks after a single step of training \cite{Ba2022, dandi2023twolayer, moniri2024theory, cui2024asymptotics} provides a first clue. These works show that with an aggressive learning rate the hidden-layer weights can be approximated by a spiked random matrix model. Investigating under which conditions the asymptotic performance is equivalent to a structured Gaussian model is an interesting venue for future research. A similar problem was studied in the context of structured inputs in \cite{pesce23a, Gerace2024}.

%% file: sections/acknowledgements.tex
We thank Florentin Guth, Nathana\"el Cuvelle-Magar, St\'ephane Mallat, Lenka Zdeborov\'a for fruitful discussions during the course of this project. We thank the Institut d'\'Etudes Scientifiques de Carg\`ese for the hospitality, where discussions for this project were held throughout the 2023 \emph{``Statistical physics
\& machine learning back
together again''} program organised by Vittorio Erba, Damien Barbier, Florent Krzakala, BL and Lenka Zdeborov\'a. BL acknowledges support from the \textit{Choose France - CNRS AI Rising Talents} program. DD is supported by ETH AI Center doctoral fellowship and ETH Foundations
of Data Science initiative. DS is supported by SNSF Ambizione Grant \texttt{PZ00P2\_209089}. HC acknowledges support from the Swiss National Science Foundation grant SMArtNet (grant number $212049$).

%% file: sections/appendix/gen_error.tex
Recall from~\Cref{ass:data+features} that we assume that the feature matrices $X,Z$ are Lipschitz-concentrated in the following sense (considering the vectors space of rectangular matrices equipped with the Frobenius norm): 
\begin{definition}[Lipschitz concentration]
    We say that a random vector $x$ in a normed vector space $\mathcal X$ is Lipschitz-concentrated with constant $\mu$ if there exists a constant $C$ such that for all $1$-Lipschitz functions $f\colon \mathcal X\to\R$ it holds that
    \begin{equation}
        \Prob*{\abs{f(x)-\E f(x)}\ge t} \le C\exp\Bigl(-\frac{t^2}{C\mu^2}\Bigr).
    \end{equation}
\end{definition}
A sufficent condition for Lipschitz concentration is that that the columns $x_i=\f(\x_i)$ are Lipschitz functions of Gaussian random vectors $\x_i$ of bounded covariance $\Omega_0:=\E \x_i\x_i^\top$, c.f.~\Cref{remark suff cond}. Indeed, let $\wt\f(\mathsf{g}):=\f(\sqrt{\Omega_0}\mathsf{g})$ and consider standard Gaussian vectors $\mathsf g_i,\ldots, \mathsf g_n$. We recall that standard Gaussian random vectors are Lipschitz-concentrated with a constant which is independent of the dimension:
\begin{theorem}[Gaussian concentration]\label{thm:gaussian_concentration}
    Let $\mathsf g$ be a random vector with independent standard Gaussian entries. Then $\mathsf g$ is Lipschitz-concentrated with constant $\mu=1$.
\end{theorem}
Therefore we can stack the Gaussian vectors $\mathsf{g}_1,\ldots,\mathsf{g}_n$ into $\mathsf g\in\R^{np}$ and write $X=X(\mathsf g)=(\wt\f(\mathsf g_1), \ldots, \wt\f(\mathsf g_n))$. Then $X$ is Lipschitz-concentrated with dimension-independent constant by~\Cref{thm:gaussian_concentration}  since for any Lipschitz $f\colon\R^{p\times n}\to\R$ it holds that $\mathsf g\mapsto f(X(\mathsf g))$ is Lipschitz due to 
\begin{equation}
    \abs{f(X(\mathsf g))-f(X(\mathsf g'))}^2 \le \norm{X(\mathsf g)-X(\mathsf g)'}_F^2 = \sum_i\norm{\wt\f(\mathsf g_i)-\wt\f(\mathsf g_i') }^2 \lesssim \sum_i \norm{\mathsf g_i-\mathsf g_i'}^2=\norm{\mathsf g-\mathsf g'}^2.
\end{equation}

\subsection*{Resolvent concentration}
It will be useful to introduce also the resolvent of the associated Gram matrix $X^\top X/p$ which is given by
\begin{equation}
    \wc G = \Bigl(\frac{X^\top X}{p}+\lambda\Bigr)^{-1}.
\end{equation}
The two resolvents are related by the identity
\begin{equation}\label{res identity XGX = 1-lambda cG}
    \frac{X^\top G X}{p} = \frac{1}{p}X^\top \Bigl( \frac{XX^\top}{p} +\lambda\Bigr)^{-1} X = \frac{X^\top X}{p} \Bigl( \frac{X^\top X}{p} +\lambda\Bigr)^{-1} = 1 - \lambda \wc G.
\end{equation}
Both resolvents $G,\wc G$ are Lipschitz-continuous with respect to the Frobenius norm due to the resovlent identity
\begin{equation}
    \begin{split}
        \Bigl(\frac{XX^\top}{p}+\lambda\Bigr)^{-1} - \Bigl(\frac{YY^\top}{p} + \lambda\Bigr)^{-1} &=\Bigl(\frac{XX^\top}{p}+\lambda\Bigr)^{-1} \frac{(Y-X)Y^\top +X(Y-X)^\top}{p} \Bigl(\frac{YY^\top}{p} + \lambda\Bigr)^{-1}\\
    \end{split}
\end{equation}
and the bound
\begin{equation}
    \norm{GX}\le \sqrt{p\norm{G}+p\lambda\norm{G^2}} \le \sqrt{2p/\lambda},
\end{equation}
implying
\begin{equation}\label{res Lipschitz}
    \norm*{G-G'}_F \le 2 \frac{\mu}{\lambda^{3/2}p^{1/2}}\norm{X-Y}_F, \quad G:=\Bigl(\frac{XX^\top}{p}+\lambda\Bigr)^{-1},\quad G':= \Bigl(\frac{YY^\top}{p} + \lambda\Bigr)^{-1}.
\end{equation}
Therefore we obtain that
\begin{equation}\label{conc res}
    \abs{\braket{A(G-\E G)}} \lesssim \frac{\braket{\abs{A}^2}^{1/2}}{\lambda^{3/2}p}, \quad \abs{\braket{A(\wc G-\E \wc G)}} \lesssim \frac{\braket{\abs{A}^2}^{1/2}}{\lambda^{3/2}p^{1/2}n^{1/2}}
\end{equation}
from~\Cref{thm:gaussian_concentration},
\begin{equation}
    \abs{\braket{A (G-G')}} \le \frac{1}{p} \norm{A}_F \norm{G-G'}_F \le \frac{1}{p} \norm{A}_F \norm{G-G'}_F \le \frac{2\braket{\abs{A}^2}^{1/2}}{\lambda^{3/2}p}\norm{X-Y}_F.
\end{equation}
and the analogous estimate for $\wc G-\wc G'$.
An important special case of~\cref{conc res} is $A$ being rank-one which yields
\begin{equation}\label{conc res xy}
    \abs{x^\top G y - \E x^\top G y} \lesssim \frac{\norm{x}\norm{y}}{\lambda^{3/2}p^{1/2}}, \quad \abs{x^\top \wc G y - \E x^\top \wc G y} \lesssim \frac{\norm{x}\norm{y}}{\lambda^{3/2}p^{1/2}}
\end{equation}

\subsection*{Quadratic form and norm concentration}
The other important concentration result needed in the proof of~\Cref{thm genRMT informal} is the concentration of quadratic forms, see e.g.\ Theorem 2.3 in~\cite{adamczak2015note}.
\begin{theorem}
\label{thm:hanson_wright}
    If $x$ is a random vector of mean zero satisfying Lipschitz concentration with constant $\mu$, and $A$ is a deterministic matrix, then
    \begin{equation}
        \abs{x^\top A x - \E x^\top A x} \lesssim \mu^2 \norm{A}_F.
    \end{equation}
\end{theorem}

Finally we need some upper bound on the operator norm of $X/\sqrt{p}$ which can be obtained standard $\epsilon$-net arguments,
\begin{equation}\label{XX norm bound}
    \norm*{\frac{XX^\top}{n}- \Omega}\prec \frac{p}{n},
\end{equation}
see e.g.\ Remark~ 5.40 in~\cite{1011.3027}.

\subsection*{Leave-one-out identities}
Define the leave-one-out resolvent $G_{-i}=(\lambda+p^{-1}\sum_{j\ne i}x_j x_j^\top)^{-1}$ for which we have the identity
\begin{equation}\label{loo identity}
    \begin{split}
        G &= G_{-i} - \frac{1}{p}\frac{G_{-i}x_i x_i^\top G_{-i}}{1+x_i^\top G_{-i} x_i/p} = G_{-i} - \lambda\frac{G_{-i}x_i x_i^\top G_{-i}}{p}\wc G_{ii} \\
        G x_i &= G_{-i} x_i \Bigl(1- \frac{1}{p}\frac{ x_i^\top G_{-i}x_i}{1+x_i^\top G_{-i} x_i/p}\Bigr)= \frac{G_{-i} x_i}{1+x_i^\top G_{-i} x_i/p} = \lambda \wc G_{ii} G_{-i}x_i
    \end{split}
\end{equation}
where the denominators can be simplified using
\begin{equation}
    -\frac{1}{1+x_i^\top G_{-i}x_i/p}=\frac{x_i^\top G_{-i}x_i}{1+x_i^\top G_{-i}x_i/p} -1 =\frac{x_i^\top G x_i}{p} -1= -\lambda (\wc G)_{ii}
\end{equation}
due to~\eqref{res identity XGX = 1-lambda cG}.

\subsection*{Anisotropic Marchenko-Pastur Law}
We are now ready to prove~\Cref{thm MP}, the anisotropic Marchenko-Pastur Law. In the comparable regime from~\Cref{thm genRMT informal} we will show that 
\begin{equation}\label{eq MP comp}
    \abs{\braket{[G(\lambda)-M(\lambda)]A}}\prec\frac{\braket{\abs{A}^2}^{1/2}}{p\lambda^3} \Bigl(1+\frac{p}{n}+\frac{n}{p}\Bigr).
\end{equation}
\begin{proof}[Proof of~\Cref{thm MP}]
For the resolvent $G$ we obtain the equation
\begin{equation}
    \begin{split}
        I &= \frac{\lambda}{p}\sum_i \Bigl( (\E \wc G_{ii}) \E G_{-i} \Omega + \E (\wc G_{ii}-\E \wc G_{ii}) G_{-i} x_i x_i^\top\Bigr) + \lambda \E G \\
        &=  \E G \Bigl(\lambda\frac{n}{p}\braket{\E\wc G}  \Omega + \lambda\Bigr) + \frac{\lambda}{p} \sum_i \Bigl(\braket{\E \wc G} (\E G_{-i}-\E G) \Omega  + \E( \wc G_{ii}-\E \wc G_{ii})G_{-i} x_i x_i^\top\Bigr) \\
    \end{split}
\end{equation}
so that
\begin{equation}
    \E G = \Bigl(\lambda\frac{n}{p}\braket{\E\wc G}\Omega + \lambda\Bigr)^{-1} + \frac{\lambda}{p} \sum_i \Bigl(\braket{\E \wc G} (\E G_{-i}-\E G) \Omega  + \E( \wc G_{ii}-\E \wc G_{ii})G_{-i} x_i x_i^\top\Bigr) \Bigl(\lambda\frac{n}{p}\braket{\E\wc G}\Omega + \lambda\Bigr)^{-1}.
\end{equation}
Using the bounds
\begin{equation}
    \norm{G_{-i} x_i x_i^\top-\E G_{-i} x_i x_i^\top}_F \le \norm{G_{-i} x_i x_i^\top - G_{-i}\Omega }_F + \norm{ (G_{-i}- \E_{-i} G_{-i})\Omega}_F \prec \frac{1}{\lambda} + \frac{1}{p^{1/2}\lambda^{3/2}},
\end{equation}
\begin{equation}
    \norm{\E G_{-i} - \E G}_F = \lambda\abs{\wc G_{ii}}\norm*{\frac{G_{-i}x_i x_i^\top G_{-i}}{p}}_F \prec \frac{1}{p}\Bigl(\norm{G_{-i}\Omega G_{-i}}_F + \norm{G_{-i} (x_ix_i^\top-\Omega)G_{-i} }_F\Bigr)\prec \frac{1}{p^{1/2}\lambda^2}
\end{equation}
and $\abs{\wc G_{ii}-\E \wc G_{ii}} \prec \frac{1}{p^{1/2}\lambda^{3/2}}$ from~\Cref{conc res xy}
we thus obtain
\begin{equation}
    \norm*{\E G(\lambda) - M(\lambda,\braket{\E\wc G})}_F \prec \frac{n}{p^{3/2}\lambda^3},\quad M(\lambda,m):=\Bigl( \lambda\frac{n}{p}m\Omega + \lambda\Bigr)^{-1}.
\end{equation}

Note that while $M(\lambda,\braket{\E \wc G})$ is a deterministic matrix, it still depends on the expected trace of $\wc G$ explicitly. However, we claim that
\begin{equation}\label{scalar claim}
    \abs{m-\braket{\E\wc G}} \lesssim \frac{m}{p\lambda^3},
\end{equation}
proving
\begin{equation}\label{local law G}
    \norm{\E G - M}_F \lesssim \frac{1}{p^{1/2}\lambda^{3/2}} + \frac{n}{p^{3/2}\lambda^3} + \norm{M(\lambda,\braket{\E\wc G})-M(\lambda,m)}_F \lesssim \frac{1}{p^{1/2}\lambda^{3}}\Bigl(1+\frac{n}{p}+\frac{p}{n}\Bigr).
\end{equation}
Now~\Cref{eq MP comp} follows directly together with the concentration estimate~\Cref{conc res}. 
\end{proof}

\subsection*{Multi-Resolvent Deterministic Equivalents}
The key for proving~\Cref{thm genRMT informal} is extending the anisotropic Marchenko-Pastur to mutli-resolvent expressions, which we summarize in the following proposition. For simplicity we carry the precise error term in the comparable regime only in the first statement, the other ones being similar. 
\begin{proposition}\label{prop multi res}~
    \begin{enumerate}
        \item\label{GXZ} For any $A\in\R^{k\times p}$ we have\footnote{In a slight abuse of notation we use the $O(\cdots)$ notation in the sense of ``$\prec$''}
        \begin{equation}
            \frac{1}{\sqrt{kp}} \braket{GXZ^\top A} =\frac{\lambda m n}{\sqrt{kp}}\braket{M\Phi A} + \bigO*{\frac{n}{k^{1/2}p^{3/2}\lambda^3}\Bigl(1+\frac{n}{p}+\frac{p}{n}\Bigr)}
        \end{equation}
        \item\label{GOmG} For any $A\in\R^{p\times p}$ we have more generally
        \begin{subequations}
            \begin{align}
                \braket{A G \Omega G} & = \frac{\braket{ AM \Omega M}}{1- \frac{n}{p}(m\lambda)^2 \braket{\Omega M \Omega M}} + \bigO*{\frac{\braket{\abs{A}^2}^{1/2}}{p\lambda^7}}                                                                         \\
                \intertext{while for any $A,B\in\R^{p\times p}$ we have}
                \braket{A G B G}      & = \braket{ AM B M} + \frac{n}{p}(m\lambda)^2 \frac{\braket{ AM \Omega M}\braket{\Omega M B M}}{1- \frac{n}{p}(m\lambda)^2 \braket{\Omega M \Omega M}} + \bigO*{\frac{\braket{\abs{A}^2}^{1/2}\norm{B}}{p\lambda^7}}
            \end{align}
        \end{subequations}
        \item\label{XGomXG} For any $A\in\R^{p\times p}$ we have
        \begin{equation}
            \braket*{\frac{X^\top G\Omega GX A}{p}} =  \frac{\lambda^2m^2  \braket{ \Omega M \Omega M}}{1- \frac{n}{p}(m\lambda)^2 \braket{\Omega M \Omega M}}\braket{A} + \bigO*{\frac{\braket{\abs{A}^2}^{1/2}}{p\lambda^7}}
        \end{equation}
        \item\label{ZXGOmGXZ} Finally, for any $A\in\R^{p\times p}$ we have
        \begin{equation}
            \begin{split}
                \braket*{\frac{ZX^{\top}G\Omega G XZ^{\top}A}{kp}}&= (m\lambda)^2\frac{n}{k}\frac{\braket*{A\Bigl(\bigl( \Psi-2\frac{n}{p}\lambda m\Phi^\top M \Phi\bigr)\braket{ \Omega M \Omega M}+\frac{n}{p}\Phi^\top M\Omega M \Phi  \Bigr)}}{1- \frac{n}{p}(m\lambda)^2 \braket{\Omega M \Omega M}}\\
                \\ & \quad + \bigO*{\frac{\braket{\abs{A}^2}^{1/2}}{p\lambda^7}}
            \end{split}
        \end{equation}
    \end{enumerate}
\end{proposition}
Before turning to the proof of~\Cref{prop multi res}, we demonstrate how~\Cref{prop multi res} implies~\Cref{thm genRMT informal}.
\begin{proof}[Proof of~\Cref{thm genRMT informal}]
    By applying~\Cref{prop multi res} to the terms of~\Cref{eq:gen err def} we obtain
    \begin{equation}
        \begin{split}
            \mathcal E_\mathrm{gen}&=\frac{\fast^\top\Psi\fast}{k} + \frac{\fast^\top Z X^\top G\Omega GX Z^\top \fast }{k p^2}
            +  \frac{n}{p}\braket*{\frac{X^\top G\Omega GX\Sigma}{p}} -2 \frac{\fast^\top \Phi^\top G X Z^\top \fast}{kp}\\
            &= \frac{1}{k}\fast\biggl(\Psi + (m\lambda)^2\frac{n}{p}\frac{\bigl( \Psi-2\frac{n}{p}\lambda m\Phi^\top M \Phi\bigr)\braket{ \Omega M \Omega M}+\frac{n}{p}\Phi^\top M\Omega M \Phi  }{1- \frac{n}{p}(m\lambda)^2 \braket{\Omega M \Omega M}} - 2\lambda m\frac{ n}{p}\Phi^\top M\Phi \biggr)\fast\\&\qquad +\braket{\Sigma}  \frac{ (\lambda m)^2\frac{n}{p} \braket{ M\Omega M\Omega  }}{1-\frac{n}{p}(\lambda m)^2\braket{\Omega M\Omega M}}+ \bigO*{\frac{\norm{\fast}^2}{p^{1/2}\lambda^7}}.
        \end{split}
    \end{equation}
    It remains to show that the matrix in the brackets can be simplified to the expression in~\Cref{thm genRMT informal}. For the last term in the numerator of the fraction we use
    \begin{equation}
        m\lambda \frac{n}{p} M\Omega M= M -\lambda M^2,
    \end{equation}
    so that the bracket, after simplifying, becomes
    \begin{equation}
        \begin{split}
            \frac{\Psi -m\lambda\frac{n}{p}\Phi^\top (M+\lambda M^2) \Phi  }{1- \frac{n}{p}(m\lambda)^2 \braket{\Omega M \Omega M}},
        \end{split}
    \end{equation}
    just as claimed.
\end{proof}
\begin{proof}[Proof of~\Cref{prop multi res}]
    We begin with the proof of~\Cref{GXZ}. First note that $\braket{GXZ^\top A}$ is a Lipschitz function of the Gaussian randomness $d$ used to construct $X$ and $Z$. Indeed, denoting $G,X,Z$ evaluated at another realization of the Gaussian randomness by $G',X',Z'$ we have
    \begin{equation}
        \begin{split}
            \braket{GXZ^\top A} - \braket{G' X' (Z')^\top A} &= \braket{(G-G')XZ^\top A} + \braket{G'(X-X')Z^\top A}+\braket{G' X' (Z-Z')^\top A} \\
            &= \bigO*{\frac{\norm{X-X'}_F \norm{X}\norm{Z}\braket{\abs{A}^2}^{1/2}}{\lambda^{3/2}p} + \frac{(\norm{X-X'}_F\norm{Z}+\norm{X}\norm{Z-Z'}_F)\braket{\abs{A}^2}^{1/2}}{p\lambda } },
        \end{split}
    \end{equation}
    so that on the high probability event (recall~\Cref{XX norm bound}) that $\norm{X}\prec \sqrt{p}, \norm{Z}\prec\sqrt{k}$ it follows that $\braket{GXZ^\top A}$ is Lipschitz with constant $\braket{\abs{A}^2}^{1/2}/p\lambda^{3/2}$. By estimating the complement of this high probability event trivially we can conclude
    \begin{equation}
        \abs*{\frac{1}{\sqrt{kp}}\braket{GXZ^\top A}- \frac{1}{\sqrt{kp}}\braket{\E GXZ^\top A}} \prec \frac{\braket{\abs{A}^2}^{1/2}}{p\lambda^{3/2}}.
    \end{equation}
    For the expectation we write out $XZ^\top$ and use~\cref{loo identity} we obtain
    \begin{equation}
        \begin{split}
            \frac{1}{\sqrt{kp}}G XZ^{\top} &=\frac{1}{\sqrt{kp}}\sum_i Gx_iz_i^\top=\frac{1}{\sqrt{kp}}\sum_i \lambda\wc G_{ii} G_i x_iz_i^\top.
        \end{split}
    \end{equation}
    With
    \begin{equation}
        \begin{split}
            \frac{\lambda}{\sqrt{kp}}\E\sum_i (\wc G)_{ii} \braket{G_i x_i z_i^\top A} &= \frac{\lambda}{\sqrt{kp}}\sum_i \Bigl((\E\wc G_{ii}) \braket{\E G_i x_i z_i^\top A} + \bigO*{\sqrt{\Var \wc G_{ii}}\sqrt{\Var \braket{G_i x_i z_i^\top A}}}\Bigr)\\
            &= \frac{\lambda}{\sqrt{kp}}\sum_i (\E\wc G_{ii}) \braket{\E G_i \Phi A} + \bigO*{\frac{n}{k^{1/2}p^{3/2}\lambda^2}}\\
            & = \frac{\lambda m n}{\sqrt{kp}}\braket{M\Phi A} + \bigO*{\frac{n}{k^{1/2}p^{3/2}\lambda^3}\Bigl(1+\frac{n}{p}+\frac{p}{n}\Bigr)}
        \end{split}
    \end{equation}
    due to~\Cref{local law G}, $\Var \wc G_{ii}\lesssim \frac{1}{p\lambda^{3}}$ and
    \begin{equation}
        \Var \braket{G_{-i} x_i z_i^\top A} \lesssim \frac{1}{p^2}\E_{-i} \norm{A G_{-i}}_F^2 + \Var_{-i} \braket{ G_{-i}\Phi A} \lesssim \frac{\braket{\abs{A}^2}}{p\lambda^3}
    \end{equation}
    by~\cref{conc res}, this concludes the proof of~\Cref{GXZ}.

    We now turn to the proof of~\Cref{GOmG}. First note that by Lipschitz concentration we have
    \begin{equation}\label{AGBG lipschitz}
        \abs{\braket{AGBG-\E AGBG}} \lesssim \frac{\norm{A}\braket{\abs{B}^2}^{1/2}}{p\lambda^{5/2}}
    \end{equation}
    due to
    \begin{equation}
        \abs{\braket{A G B G}- \braket{AG' B G' } }\le \abs{\braket{A (G-G')BG}} + \abs{\braket{AG'B(G-G')}} \le 2\frac{\norm{A}\norm{B}_F}{p\lambda} \norm{G-G'}_F
    \end{equation}
    and~\cref{res Lipschitz}.

    It is useful to expand $G$ around $M$ as in
    \begin{equation}
        \begin{split}
            G= M+\lambda M \Omega G\frac{n}{p}\braket{m-\wc G} -  M\frac{XX^\top}{p}G+ \lambda M\Omega G \frac{n}{p}\braket{\wc G} = M-  M\frac{XX^\top}{p}G+ \lambda \braket{\wc G} M\Omega G \frac{n}{p} + \bigO*{ \frac{1}{p\lambda^3} } M\Omega G
        \end{split}
    \end{equation}
    using~\cref{scalar claim} in the second step. Consequently we obtain
    \begin{equation}\label{GAGB exp}
        \begin{split}
            \braket{GAGB} &= \braket{MAGB} - \braket{M\frac{XX^\top}{p}G AGB} + \frac{n\lambda}{p}\braket{\wc G} \braket{M\Omega G AGB} + \bigO*{\frac{\braket{\abs{A}^2}^{1/2}\braket{\abs{B}^2}^{1/2}}{p \lambda^6}}\\
            & =\braket{MAMB} -\frac{1}{p} \sum_i (\braket{M x_i x_i^\top G AG B}-\lambda\wc G_{ii}\braket{M\Omega GAGB})  + \bigO*{\frac{\norm{B}\braket{\abs{A}^2}^{1/2}}{p\lambda^6}} \\
            & = \braket{MAMB} -\frac{\lambda}{p} \sum_i \wc G_{ii} (\braket{M x_i x_i^\top G_{-i} AG_{-i} B}-\braket{M\Omega GAGB})  + \bigO*{\frac{\norm{B}\braket{\abs{A}^2}^{1/2}}{p\lambda^6}}\\
            &\quad + \frac{\lambda^2}{p} \sum_i \wc G_{ii}^2 \frac{x_i^\top G_{-i} AG_{-i}x_i}{p}\frac{x_i^\top G_{-i} B M x_i}{p},
        \end{split}
    \end{equation}
    using~\cref{loo identity} in the third step. The second term of~\cref{GAGB exp} can be estimated in expectation using
    \begin{equation}
        \begin{split}
            \frac{\lambda}{p}\E\sum_i\wc G_{ii} \braket{ M x_i x_i^\top G_{-i} AG_{-i} B} &= \frac{\lambda}{p}\sum_i\Bigl((\E\wc G_{ii}) \braket{\E M\Omega G_{-i} A G_{-i}B} + \bigO*{\sqrt{\Var \wc G_{ii}}\sqrt{\Var \braket{ M x_i x_i^\top G_{-i} AG_{-i} B}}\Bigr)}\\
            &=\frac{\lambda}{p}\sum_i(\E\wc G_{ii}) \braket{\E M\Omega G A GB} + \bigO*{\frac{n\norm{B}\braket{\abs{A}^2}^{1/2}}{p^2\lambda^{4}} } \\
            &= \frac{\lambda}{p}\E\sum_i \wc G_{ii} \braket{ M\Omega G A GB} + \bigO*{\frac{n\norm{B}\braket{\abs{A}^2}^{1/2}}{p^2\lambda^{4}} }
        \end{split}
    \end{equation}
    since $\Var \wc G_{ii}\lesssim \frac{1}{p\lambda^{3}}$,
    \begin{equation}
        \Var\braket{ M x_i x_i^\top G_{-i} A G_{-i} B } \lesssim \frac{1}{p^2}\E_{-i} \norm{G_{-i}A G_{-i} B M}_F^2 + \Var_{-i} \braket{M\Omega G_{-i} A G_{-i}B} \lesssim\frac{\norm{B}^2\braket{\abs{A}^2}}{p\lambda^6}(1+\frac{1}{p\lambda}).
    \end{equation}
    and
    \begin{equation}
        \norm{G-G_i} \lesssim \frac{1}{p\lambda^2}, \quad \braket{M\Omega G_{-i} A G_{-i} B}= \braket{M\Omega G A G B} + \bigO*{\frac{\norm{B}\braket{\abs{A}^2}^{1/2}}{p\lambda^4}}.
    \end{equation}
    For the last term of~\Cref{GAGB exp} we have
    \begin{equation}
        \begin{split}
            \frac{x_i^\top G_{-i} AG_{-i}x_i}{p} &= \braket{\Omega G_{-i} A G_{-i}} + \bigO*{\frac{1}{p}\norm{G_{-i}AG_{-i}}_F}= \braket{\Omega G A G} + O (\frac{\braket{\abs{A}^2}^{1/2}}{p^{1/2}\lambda^2})\\
            \frac{x_i^\top G_{-i} B M x_i}{p} &= \braket{\Omega G_{-i} B M} + \bigO*{\frac{1}{p}\norm{G_{-i}BM}_F}= \braket{\Omega M B M} + O (\frac{\braket{\abs{B}^2}^{1/2}}{p^{1/2}\lambda^2}),
        \end{split}
    \end{equation}
    so that with
    \begin{equation}
        \begin{split}
            \E\wc G_{ii}^2 \frac{x_i^\top G_{-i} AG_{-i}x_i}{p}\frac{x_i^\top G_{-i} B M x_i}{p} &= (\E \wc G_{ii}^2) \Bigl(\E \frac{x_i^\top G_{-i} A G_{-i}x_i}{p}\Bigr) \Bigl(\E \frac{x_i^\top G_{-i} B M x_i}{p}\Bigr) + \bigO*{\frac{\braket{\abs{A}^2}^{1/2}\braket{\abs{B}^2}^{1/2}}{p\lambda^7}}\\
            &=(\E \wc G_{ii}^2) \braket{\E \Omega G A G} \braket{ \E \Omega G B M} + \bigO*{\frac{\braket{\abs{A}^2}^{1/2}\braket{\abs{B}^2}^{1/2}}{p\lambda^7}}
        \end{split}
    \end{equation}
    and
    \begin{equation}
        \frac{1}{p}\sum_i \wc G_{ii}^2 = \frac{n}{p}m^2 +2 \frac{n}{p}m \braket{\wc G-m} + \frac{1}{p}\sum_i (\wc G_{ii}-m)^2= \frac{n}{p}m^2 +\bigO*{\frac{1}{p\lambda^5} + \frac{n}{p^2\lambda^3}}
    \end{equation}
    we arrive at
    \begin{equation}
        \frac{\lambda^2}{p} \E\sum_i \wc G_{ii}^2 \frac{x_i^\top G_{-i} AG_{-i}x_i}{p}\frac{x_i^\top G_{-i} B M x_i}{p} = \frac{n}{p}(m\lambda)^2 \braket{\Omega M B M} \braket{\E \Omega G A G} + \bigO*{\frac{\braket{\abs{A}^2}^{1/2}\braket{\abs{B}^2}^{1/2}}{p\lambda^6}}.
    \end{equation}
    Choosing $B=\Omega$ it follows that
    \begin{equation}
        \braket{GA G \Omega}(1 - \frac{n}{p}\lambda^2 m^2 \braket{\Omega M \Omega M}) = \braket{MA M \Omega}   + \bigO*{\frac{\braket{\abs{A}^2}^{1/2}}{p\lambda^6}},
    \end{equation}
    so that the final claim~\Cref{GOmG} follows upon division.

    Turning to the proof of~\Cref{XGomXG} we first note that by~\cref{loo identity} we have
    \begin{equation}
        \begin{split}
            \E\Bigl(\frac{X^\top G\Omega GX}{p}\Bigr)_{ii} &= \lambda^2 \E\wc G_{ii}^2 \braket{\E\Omega G_{-i}\Omega G_{-i}} + \bigO*{\frac{1}{p}\sqrt{\Var \wc G_{ii}^2 } \sqrt{\Var x_i^\top G_{-i}\Omega G_{-i} x_i}}\\
            &= \lambda^2 \E \wc G_{ii}^2 \frac{\braket{ \Omega M \Omega M}}{1- \frac{n}{p}(m\lambda)^2 \braket{\Omega M \Omega M}} + \bigO*{\frac{1}{p\lambda^7}},
        \end{split}
    \end{equation}
    so that by a Lipschitz concentration argument as in~\Cref{AGBG lipschitz} we obtain for the diagonal part $A_d$ of $A=A_d+A_o$ that
    \begin{equation}
        \braket*{\frac{X^\top G\Omega GX}{p}A_d} = \frac{\lambda^2m^2  \braket{ \Omega M \Omega M}}{1- \frac{n}{p}(m\lambda)^2 \braket{\Omega M \Omega M}}\braket{A_d} + \bigO*{\frac{\braket{\abs{A_d}^2}^{1/2}}{p\lambda^7}}.
    \end{equation}
    For the off-diagonal part we use~\cref{loo identity} twice to obtain
    \begin{equation}
        \begin{split}
            \Bigl(\frac{X^\top G\Omega GX}{p}\Bigr)_{ij} &= \frac{\lambda^{2}\wc G_{ii}\wc G_{jj}}{p} x_i^\top G_{-i}\Omega G_{-j} x_j\\
            & = \frac{\lambda^{2}\wc G_{ii}\wc G_{jj}}{p} x_i^\top G_{-ij}\Omega G_{-ij} x_j + \frac{\lambda^{4}\wc G_{ii}^2\wc G_{jj}^2}{p^3} x_i^\top G_{-ij} x_jx_j^\top G_{-ij}\Omega G_{-ij} x_ix_i^\top G_{-ij} x_j\\
            &\quad - \frac{\lambda^{3}\wc G_{ii}^2\wc G_{jj}}{p^2} x_i^\top G_{-ij}\Omega G_{-ij}x_ix_i^\top G_{-ij} x_j  - \frac{\lambda^{3}\wc G_{ii}\wc G_{jj}^2}{p^2} x_i^\top G_{-ij}x_jx_j^\top G_{-ij}\Omega G_{-ij} x_j.
        \end{split}
    \end{equation}
    The second term can be estimated trivially by $p^{-3/2}\lambda^{-4}$, while for the first, third and fourth terms the trivial estimates of $p^{-1/2}\lambda^{-2}$, $p^{-1} \lambda^{-3}$ and $p^{-1/2}\lambda^{-3}$ do not suffice. For those we use the expectation and decompose $\wc G_{ii}=m+(\wc G_{ii}-m)$, $\wc G_{jj}=m+(\wc G_{jj}-m)$ to obtain
    \begin{equation}
        \E \frac{\lambda^{2}\wc G_{ii}\wc G_{jj}}{p} x_i^\top G_{-ij}\Omega G_{-ij} x_j= \E \frac{\lambda^{2}(\wc G_{ii}-m)(\wc G_{jj}-m)}{p} x_i^\top G_{-ij}\Omega G_{-ij} x_j = \bigO*{\frac{1}{\lambda^2 p^{3/2}}}
    \end{equation}
    and
    \begin{equation}
        \E\frac{\lambda^{3}\wc G_{ii}^2\wc G_{jj}}{p^2} x_i^\top G_{-ij}\Omega G_{-ij}x_ix_i^\top G_{-ij} x_j = \E \frac{\lambda^{3}(\wc G_{ii}^2\wc G_{jj}-m^3)}{p^2} x_i^\top G_{-ij}\Omega G_{-ij}x_ix_i^\top G_{-ij} x_j = \bigO*{\frac{1}{p^{3/2}\lambda^{7/2}}}
    \end{equation}
    using that, say, $x_j$ is centered and independent of $x_i, G_{-ij}$. By combining these estimates we obtain
    \begin{equation}
        \E \abs*{\Bigl(\frac{X^\top G\Omega GX}{p}\Bigr)_{ij}} = \bigO*{\frac{1}{p^{3/2}\lambda^4}},
    \end{equation}
    concluding the proof of~\Cref{XGomXG}.

    We now turn to the proof of~\eqref{ZXGOmGXZ} which follows a similar strategy as the proof of~\Cref{GOmG}. First we note that by a Lipschitz concentration argument as in~\Cref{AGBG lipschitz} it is sufficient to approximate the expectation of $ZX^\top G \Omega G X Z^\top$.
    By writing out $ZX^\top$ and $X Z^\top$ and using~\cref{loo identity} twice we obtain
    \begin{equation}\label{ZXGGXZ exp}
        \begin{split}
            \frac{1}{kp}ZX^{\top}G\Omega G XZ^{\top} &= \frac{1}{kp}\sum_{ij} z_i x_i^\top G \Omega G x_j z_j^\top\\
            &= \frac{1}{kp} \sum_i (\lambda \wc G_{ii})^2 z_i x_i^\top G_{-i} \Omega G_{-i} x_i z_i^\top + \frac{1}{kp} \sum_{i\ne j} (\lambda \wc G_{ii})(\lambda \wc G_{jj}) z_i x_i^\top G_{-i} \Omega G_{-j} x_j z_j^\top.
        \end{split}
    \end{equation}
    For the first term of~\Cref{ZXGGXZ exp} we have
    \begin{equation}
        \begin{split}
            \frac{n}{kp}\E\braket{Az_i x_i^\top G_{-i} \Omega G_{-i} x_i z_i^\top} &= \frac{n}{k^2p} (\E z_i^\top A z_i) (\E x_i^\top G_{-i} \Omega G_{-i} x_i)+\bigO*{\frac{n}{k^2p} \sqrt{\Var z_i^\top A z_i}\sqrt{\Var x_i^\top G_{-i} \Omega G_{-i} x_i} }\\
            &= \frac{n}{k} \braket{A\Psi} \E \braket{\Omega G_{-i}\Omega G_{-i}} + \bigO*{\frac{n\braket{\abs{A}^2}^{1/2}}{p^{1/2}k^{3/2}\lambda^2}}\\
            &= \frac{n}{k} \braket{A\Psi}\frac{\braket{ \Omega M \Omega M}}{1- \frac{n}{p}(m\lambda)^2 \braket{\Omega M \Omega M}} + \bigO*{\frac{n\braket{\abs{A}^2}^{1/2}}{pk\lambda^3}(1+\sqrt{p/k})}
        \end{split}
    \end{equation}
    using~\Cref{GOmG} in the ultimate step. For the second term in the right hand side of~\Cref{ZXGGXZ exp} we expand both $G_{-i}$ and $G_{-j}$ around $G_{-ij}$ using~\Cref{loo identity} to
    \begin{equation}\label{xxGOmGxxA}
        \begin{split}
            \braket{z_i x_i^\top G_{-i} \Omega G_{-j} x_j z_j^\top A} &\approx \braket*{z_i x_i^\top \Bigl(G_{-ij}- \lambda m G_{-ij}\frac{x_jx_j^\top}{p}G_{-ij}\Bigr) \Omega \Bigl(G_{-ij}-m\lambda G_{-ij}\frac{x_ix_i^\top}{p}G_{-ij}\Bigr) x_j z_j^\top A}\\
            & = \braket*{z_i x_i^\top G_{-ij} \Omega G_{-ij} x_j z_j^\top A} + (\lambda m)^2 \braket*{z_i x_i^\top  G_{-ij}\frac{x_jx_j^\top}{p}G_{-ij} \Omega  G_{-ij}\frac{x_ix_i^\top}{p}G_{-ij} x_j z_j^\top A}\\
            &\quad - \lambda m \braket*{z_i x_i^\top G_{-ij}\Omega G_{-ij}\frac{x_ix_i^\top}{p}G_{-ij} x_j z_j^\top A} - \lambda m \braket*{z_i x_i^\top G_{-ij}\frac{x_jx_j^\top}{p}G_{-ij} \Omega G_{-ij} x_j z_j^\top A}.
        \end{split}
    \end{equation}
    Here in the first line we replaced $(\wc G_{-i})_{jj}$ and $(\wc G_{-j})_{ii}$ by $m$ which results in an error term negligible compared to the other error terms. The first term of~\Cref{xxGOmGxxA} can, in expectation, be approximated by
    \begin{equation}
        \E\braket*{z_i x_i^\top G_{-ij} \Omega G_{-ij} x_j z_j^\top A} = \E\braket{\Phi^\top G_{-ij} \Omega G_{-ij} \Phi A} = \frac{\braket{\Phi^\top M\Omega M \Phi A}}{1-\frac{n}{p}(m\lambda)^2 \braket{\Omega M \Omega M}} + \bigO*{\frac{\braket{\abs{A}^2}^{1/2}}{p\lambda^7}},
    \end{equation}
    using~\Cref{GOmG} in the ultimate step. The third term of~\Cref{xxGOmGxxA} can be approximated by
    \begin{equation}
        \begin{split}
            &\lambda m \E \braket*{z_i x_i^\top G_{-ij}\Omega G_{-ij}\frac{x_ix_i^\top}{p}G_{-ij} x_j z_j^\top A}\\
            &\quad= \frac{1}{kp}\lambda m(x_i^\top G_{-ij} \Phi Az_i) (x_i^\top G_{-ij}\Omega G_{-ij} x_i)\\
            &\quad = \lambda m \E_{-ij}\Bigl(\braket{\Phi^\top G_{-ij}\Phi A} \braket{\Omega G_{-ij}\Omega G_{-ij}} + \bigO*{\sqrt{\Var_i \frac{x_i^\top G_{-ij} \Phi Az_i}{k}}\sqrt{\Var_i \frac{x_i^\top G_{-ij}\Omega G_{-ij} x_i}{p}} }\Bigr) \\
            &\quad= \lambda m \frac{\braket{\Phi^\top M \Phi A} \braket{\Omega M\Omega M}}{1-\frac{n}{p}(m\lambda)^2 \braket{\Omega M \Omega M}} + \bigO*{\frac{\braket{\abs{A}^2}}{\lambda^{7}p}\Bigl(1+\sqrt{\frac{p}{k}}\Bigr)}
        \end{split}
    \end{equation}
    and the fourth term is exactly the same by symmetry. Here in the ultimate step we used
    \begin{equation}
        \Var_i \frac{x_i^\top G_{-ij}\Omega G_{-ij} x_i}{p} \lesssim \frac{1}{p^2}\norm{G_{-ij}\Omega G_{-ij}}_F^2\lesssim\frac{1}{p\lambda^2}, \quad \Var_i \frac{x_i^\top G_{-ij}\Phi A z_i}{k}\lesssim \frac{\braket{\abs{A}^2}}{\lambda k}
    \end{equation}
    and~\Cref{local law G,GOmG}.
    Finally, for the second term of~\Cref{ZXGGXZ exp} we use the simple bound
    \begin{equation}
        \begin{split}
            &\braket*{z_i x_i^\top  G_{-ij}\frac{x_jx_j^\top}{p}G_{-ij} \Omega  G_{-ij}\frac{x_ix_i^\top}{p}G_{-ij} x_j z_j^\top A} \\
            &\quad = \frac{1}{kp^2} (x_i^\top  G_{-ij} x_j)(x_j^\top G_{-ij} \Omega  G_{-ij} x_i) (x_i^\top G_{-ij} x_j) (z_j^\top A z_i) \\
            & \quad = \bigO*{\frac{1}{kp^2} \norm{G_{-ij}}_F \norm{G_{-ij}\Omega G_{-ij}}_F \norm{G_{-ij}}_F \norm{A}_F} = \bigO*{\frac{\braket{\abs{A}^2}^{1/2}}{k^{1/2}p^{1/2}\lambda^4}}.
        \end{split}
    \end{equation}
    By combining all the above estimates we conclude the proof of~\Cref{ZXGOmGXZ}.
\end{proof}

%% file: sections/appendix/app_linearization.tex
\subsection{Technical background}
\label{app:technical}
In this section we state several definition and propositions from~\cite{nourdin2012normal}, that will be used further in our arguments.
Let \(x \in \R^d\) be a mean-zero Gaussian vector with covariance \(\E x x^T = I\).
Let \(X = \{X(v) \coloneqq v^{\top} x, \text{ for } v \in \R^d\}\) be a collection of jointly Gaussian centered random variables.
Note that \(\E X(g) X(h) = g^{\top} h\).
The theory of \textit{Wiener chaos}, which will be introduced shortly, 
can be used to study functions on the probability space \((\Omega, \mathcal{F}, P)\), 
where \(\mathcal{F}\) is generated by \(X\). For our needs, we only state the results for the explicit construction of \(X\), however, note that the results from~\cite{nourdin2012normal} are about general separable Hilbert spaces.

Following (\cite{nourdin2012normal}, Definition 2.2.3),
we write \(\cH_n\) to denote the closed linear subspace of \(L^2(\Omega, \mathcal{F}, P)\) 
generated by the random variables of type \(H_n(X(h)), h \in \R^d\), \(\norm{h} = 1\), where \(H_n\) is the \(n\)-th \emph{Hermite polynomial}. We call \(\cH_n\), the \(n\)-th Wiener chaos.



\begin{definition}
    Let \(L^2(\Omega, \mathfrak{H}^{\tilde \otimes p})\) be the space of functions \(f: \R^{d \times p} \to \R\), such that \(f\) is square-integrable and
    \begin{equation}
        f(a_1, \ldots, a_p) = \frac{1}{p!} \sum_{\sigma \in S_p} f(a_{\sigma(1)}, \ldots, a_{\sigma(p)}).
    \end{equation}
\end{definition}
Let \(\sS\) denote the set of all random variables of the form \(f(X(h_1), \ldots, X(h_m))\),
where \(f: \R^{m} \to \R\) is a \(C^{\infty}\)-function. 
\begin{definition}[\cite{nourdin2012normal}, Definition 2.3.2]
    Let \(F \in \sS\) and \(p \geq 1\) be an integer. The \(p\)th Malliavin derivative of \(F\) (with respect to \(X\))
    is the element of \(L^2(\Omega, \mathfrak{H}^{\tilde \otimes p})\), defined by
    \begin{equation}
        D^p F \coloneqq \sum_{i_1, \ldots, i_p = 1}^{m} \frac{\partial^p f}{\partial x_{i_1} \ldots \partial x_{i_p}} (X(h_1), \ldots, X(h_m)) h_{i_1} \otimes \ldots \otimes h_{i_p}.
    \end{equation}
\end{definition}
\begin{proposition}[\cite{nourdin2012normal}, Proposition 2.3.7]
    Let \(\phi: \R^m \to \R\) be a continuously differentiable function with bounded partial derivatives.
    Suppose that \(F = (F_1, \ldots, F_m)\) is a random vector whose components are functions with derivatives in \(L^q(\gamma)\), for some \(q \geq 1\).
    Then, derivative of \(\phi(F)\) also lies in \(L^q(\gamma)\) and
    \begin{equation}
        D \phi(F) = \sum_{i = 1}^m \frac{\partial \phi}{\partial x_i}(F) D F_i.
    \end{equation}
\end{proposition}
\begin{definition}[\cite{nourdin2012normal}, Definition 2.5.2]
We define $\delta^p u$ as the unique element of $L^2$ satisfying 
\[\E[F\delta^p(u)]= E[\braket{D^p F,u}_{\mathfrak H^{\otimes p}}].\]
\end{definition}
\begin{definition}[\cite{nourdin2012normal}, Definition 2.7.1]
    Let \(p \geq 1\) and \(f \in \mathfrak{H}^{\tilde \otimes p}\). The \(p\)th multiple integral of \(f\) with respect to \(X\) is defined by
    \(I_p(f) = \delta^p(f)\).
\end{definition}
\begin{proposition}[\cite{nourdin2012normal}, Proposition 2.7.5]
    Fix integers \(1 \leq q \leq p\) and \(f \in \mathfrak{H}^{\tilde \otimes p}\) and \(g \in \mathfrak{H}^{\tilde \otimes q}\).
    We have
    \begin{equation}
        \E I_p(f) I_q(g) = \delta_{pq} p! \braket{f, g}_{\mathfrak{H}^{\otimes p}}
    \end{equation}
\end{proposition}
\begin{theorem}[\cite{nourdin2012normal}, Theorem 2.7.7]
    Let \(f \in \mathfrak{H}\) be such that \(\norm{f}_{\mathfrak{H}} = 1\). Then, for any integer \(p \geq 1\), we have
    \begin{equation}
        H_p(X(f)) = I_p(f^{\otimes p}),
    \end{equation}
    where $H_p$ is the $p$-th Hermite polynomial. 
\end{theorem}
\begin{corollary}[\cite{nourdin2012normal}, Corollary 2.7.8]
\label{cor:wiener_chaos_general}
    Every \(F \in L^2(\Omega)\) can be expanded as 
    \begin{equation}
        F = \E F + \sum_{p = 1}^{\infty} I_p(f_p),
    \end{equation}
    for some unique collection of kernels \(f_p \in \mathfrak{H}^{\tilde \otimes p}\), \(p \geq 1\).
    Moreover, if \(F \in C^{\infty}\), then for all \(p \geq 1\), 
    \begin{equation}
        f_p = \frac{1}{p!}\E D^{p} F.
    \end{equation}
\end{corollary}
\begin{theorem}[\cite{nourdin2012normal}, Theorem 5.1.5]
\label{thm:clt}
    Let \(F \in C^{\infty}\) be a square-integrable function. Let \(\E F = 0\) and \(\E F^2 = \sigma^2 > 0\) and \(N \sim \cN(0, \sigma^2)\).
    Let \(h: \R \to \R\) be \(C^2\) with \(\norm{h''}_{\infty} < \infty\).
    Then,
    \begin{equation}
        \abs{\E h(N) - \E h(F)} \leq \frac{1}{2} \norm{h''}_{\infty} \E \left[ \abs{\braket{DF, -D L^{-1} F} - \sigma^2}\right].
    \end{equation}
\end{theorem}
Finally, we use the following multivariate version of the previous theorem.
\begin{theorem}[\cite{nourdin2012normal}, Theorem 6.1.2]
\label{thm:multi-clt}
    Fix \(c \geq 2\), and let \(F = (F_1, \ldots, F_c)\) be a random vector such that \(F_i \in \mathbb{D}^{1, 4}\)
    with \(\E F_i = 0\) for any \(i\).
    Let \(C \in \mathcal{M}_c(\R)\) be a symmetric non-negative definite matrix,
    and let \(N \sim \cN(0, C)\).
    Then, for any \(h\: : \: \R^c \to \R\) belonging to \(\mathcal{C}^2\)
    such that \(\norm{h''}_{\infty} < \infty\),
    \begin{equation}
        \abs{\E h(F) - \E h(N)} \leq \frac{c}{2} \norm{h''}_{\infty} \sqrt{\sum_{i,j=1}^c \E \left[\left(C_{ij} - \braket{D F_j, -DL^{-1}F_i}_{\mathfrak{H}}\right)^2\right]}
    \end{equation}
\end{theorem}
\begin{remark}
    We believe there is a mistake in the original formulation of Theorem 6.1.2 in~\cite{nourdin2012normal}.
    In particular, originally the expression on the right hand side did not contain \(c\) term.
\end{remark}
For our application, we need the following expansion: for smooth odd functions \(f\), and matrix \(W \in \R^{k \times d}\), we can write
\begin{equation}
\label{eq:wiener_chaos_expansion}
\begin{aligned}
    f(Wx)_i = f(w_i^{\top} x) = \sum_{p \geq 1} \frac{\E f^{(p)} ((W W^{\top})^{1/2}_{ii} N)}{p!} I_p(w_i^{\otimes p}), 
\end{aligned}
\end{equation}
where \(w_i \in \R^d\) is the \(i\)-th row of \(W\). Here without loss of generality we assume that $x$ has i.i.d.\ entries, the general case of covariance $\Omega_0$ then follows upon redefining $W_1\mapsto W_1\sqrt{\Omega_0}$. 
Let \((f_{\ell}): \R \to \R\) be a sequence of smooth functions, 
    \((W^{\ell})\) be a sequence of matrices.
    We define a sequence of vectors \(x^i\), such that \(x^0 \coloneqq x \sim \cN(0, I)\), 
    \(x^{\ell + 1} = f_{\ell + 1}(W^{\ell + 1} x^{\ell}).\)
\begin{lemma}[Weak correlation]
\label{lemma:weak_corr}
Let \(b \geq 1\) be a fixed integer. Let \(h_0, h_1, \ldots, h_b\) be a collection of functions. Then, we have that
    \begin{equation}
        \E \left[ h_0(u^{\top} f_1(W^1 x)) \prod_{i=1}^b h_i(w_{i}^{\top} x)\right] =
        \E h_0(u^{\top} f_1(W^1 x)) \prod_{i=1}^b \E h_i(w_{i}^{\top} x) + O(d^{-1/2}).
    \end{equation}
\end{lemma}
\begin{proof}
    The fact that \(u_i \lesssim d^{-1/2}\) and \(f_1(w^{\top} x) \lesssim 1\) together with perturbation analysis imply that
    \begin{equation}
        \E \left[ h_0(u^{\top} f_1(W^1 x)) \prod_{i=1}^b h_i(w_{i}^{\top} x)\right] = 
        \E \left[ h_0\left(\sum_{k\geq b + 1} u_k f_1(w_k^\top x)\right) \prod_{i=1}^b h_i(w_{i}^{\top} x)\right] + O(d^{-1/2}).
    \end{equation}
    Let \(A \coloneqq h_0\left(\sum_{k\geq b + 1} u_k f_1(w_k^\top x)\right)\) and \(B \coloneqq \prod_{i=1}^b h_i(w_{i}^{\top} x)\). Note that for any \(p \geq 1\), \(\braket{\E D^p A, \E D^p B}\) constitutes of products of \(\braket{w_i, w_j}\), where \(i \neq j\). Each of these products is of order \(O(d^{-1/2})\) by our assumptions. Therefore, in total, \(\braket{\E D^p A, \E D^p B} = O(d^{-p/2})\). This implies that 
    \begin{equation}
        \E \left[ h_0\left(\sum_{k\geq b + 1} u_k f_1(w_k^\top x)\right) \prod_{i=1}^b h_i(w_{i}^{\top} x)\right]
        = \E \left[ h_0\left(\sum_{k\geq b + 1} u_k f_1(w_k^\top x)\right)\right] \E \left[ \prod_{i=1}^b h_i(w_{i}^{\top} x)\right] + O(d^{-1/2}).
    \end{equation}
    Similarly, it follows that \(\E \prod_{i=1}^b h_i(w_{i}^{\top} x) = \prod_{i=1}^b \E h_i(w_{i}^{\top} x) + O(d^{-1/2})\) and finally, using perturbation analysis again, we conclude that
    \begin{equation}
        \E \left[ h_0(u^{\top} f_1(W^1 x)) \prod_{i=1}^b h_i(w_{i}^{\top} x)\right] =
        \E h_0(u^{\top} f_1(W^1 x)) \prod_{i=1}^b \E h_i(w_{i}^{\top} x) + O(d^{-1/2})
    \end{equation}
\end{proof}

\subsection{One layer linearization}
Consider a mean-zero Gaussian random vector $x \in \R^d$ with covariance $\E x x^\top =I$, two weight matrices $W\in \R^{k\times d}, V\in\R^{s\times d}$ and two smooth odd functions $f,g$ applied entrywise to $Wx,Vx$. 
We assume that rows of \(W\) and \(V\) are mean-zero i.i.d. samples \((w_i, v_i) \sim (w, v)\), such that \(C_w \coloneqq \E ww^{\top}\) and \(C_v \coloneqq \E vv^{\top}\). Let \(C_{wv} = \E w v^{\top}\) if \(s = k\) and \(C_{wv} = \mathbf{0}_{d \times d}\) (all-zero matrix) otherwise.

Let \(N_w, N_v\) be jointly Gaussian mean-zero random variables, such that 
\begin{equation}
    \E N_w^2 = \Tr C_w,\quad \E N_v^2 = \Tr C_v,\quad \E N_w N_v = \Tr C_{wv}.
\end{equation}

Define 
\begin{equation}
    \begin{aligned}
        &\Phi_1 = \E f(Wx) g(Vx)^{\top}, \\
        &\Phi_1^{\mathrm{lin}} = (\E f'(N_w))(\E g'(N_v)) W V^\top + [\E f(N_w)g(N_v)-(\E f'(N_w))(\E g'(N_v))(\E N_w N_v)] I.
    \end{aligned}
\end{equation}
\begin{proposition}
\label{prop:lin_1_layer}
    We have that, with high probability, \(
        \norm{\Phi_1 - \Phi_1^{\mathrm{lin}}}_F = O(1)\).
\end{proposition}

\begin{proof}

Using a Wiener chaos expansion (\cref{eq:wiener_chaos_expansion}), we can write
\begin{equation}
    f(W x)_i = \sum_{p\ge 1} \frac{\E f^{(p)}((W W^\top)_{ii}^{1/2} N) }{p!} I_p(Wx)_i, \quad g(V x)_j = \sum_{p\ge 1} \frac{\E g^{(p)}((V V^\top)_{jj}^{1/2} N) }{p!} I_p(Vx)_j
\end{equation}
where \(N \sim \cN(0, 1)\) and $I_p(Wx),I_q(Vx)$ are random vectors with covariance
\begin{equation}
    \E I_p(Wx) I_q(Vx)^\top = p! \delta_{pq}  (W V^\top)^{\odot p}
\end{equation}
with $A^{\odot p}$ denoting the $p$-th entrywise (Hadamard) power. Thus we have the identity
\begin{equation}
    \E f(Wx)_i g(Vx)_j = \sum_{p\geq 1} \frac{1}{p!} (\E f^{(p)}((W W^\top)_{ii}^{1/2} N) ) (W  V^\top )_{ij}^{p} ( \E g^{(p)}((V V^\top)_{jj}^{1/2} N) ).
\end{equation}
From~\cref{thm:hanson_wright} (note that \(\norm{w}_{\psi_2} \sim d^{-1/2}\) and same for \(v\)), and since \(\Tr C_w  \sim 1\), it follows that 
\begin{gather}
        (W W^\top)_{ii} = \Tr C_w  + O(d^{-1/2}), \quad (V V^\top)_{jj} = \Tr C_{v} + O(d^{-1/2}), \\
        (W V^\top)_{ij} = \delta_{ij}\Tr C_{wv} + O(d^{-1/2}).
\end{gather}
From perturbation analysis, we can write
\begin{equation}
    \E f^{(p)}((W  W^{\top})_{ii}^{1/2} N) = \E f^{(p)}(\sqrt{\Tr C_w } N) + O(d^{-1/2}) = \E f^{(p)}(N_w) + O(d^{-1/2}),
\end{equation}
and similarly \(\E g^{(p)}((V  V^{\top})_{jj}^{1/2} N) = \E g^{(p)}(N_v) + O(d^{-1/2})\).
\paragraph{Off-diagonal entries}
Here, for \(p \geq 2\), we have that \((W V^{\top})^p_{ij} = O(d^{-p/2})\).
Therefore, 
\begin{equation}
    \E f(Wx)_i g(Vx)_j = \E f'(N_w) g'(N_v) (W V^{\top})_{ij} + O(d^{-1}) = (\Phi_1^{\mathrm{lin}})_{ij} + O(d^{-1}).
\end{equation}
\paragraph{Diagonal entries}
If \(s \neq k\), we have \((W  V^{\top})^p_{ii} = O(d^{-p/2})\) for \(p \geq 2\), and thus obtain the same expression as in previous case. When \(s = k\), we can rewrite the infinite sum as
\begin{equation}
\begin{aligned}
\E f(Wx)_i g(Vx)_i &= \sum_{p\geq 1} \frac{1}{p!} (\E f^{(p)}((W W^\top)_{ii}^{1/2} N) ) (W  V^\top )_{ii}^{p} ( \E g^{(p)}((V V^\top)_{ii}^{1/2} N) ) \\
    &= \sum_{p\geq 1}\frac{[\E f^{(p)}(\sqrt{\Tr C_w} N)][\E g^{(p)}(\sqrt{\Tr C_v} N)]}{p!} (\Tr C_{wv})^p + O(d^{-1/2}) \\
    &= \E f(N_w)g(N_v) + O(d^{-1/2}) = (\Phi_1^{\mathrm{lin}})_{ii} + O(d^{-1/2}).
\end{aligned}
\end{equation}
Summing up over all entries, we conclude that \(\norm{\Phi_1 - \Phi_1^{\mathrm{lin}}}_F = O(1)\).
\end{proof}
Note that in case of independent \(N_v, N_w\) (i.e., independent \(v, w\)) the second term of \(\Phi_1^{\mathrm{lin}}\) vanishes and in case of $W = V$, \(f \equiv g\) this reduces to 
\begin{equation}
    \Phi_1^{\mathrm{lin}} =  (\E f'(N_w))^2 W W^\top + [\E f(N_w)^2-(\E f'(N_w))^2\Tr C_w] I.
\end{equation}
\subsection{Two layer case}
\label{app:second_layer}
We now consider the  $2$-layer example
\begin{equation}
    f_2(W^2 f_1(W^1 x)), \quad g_2(V^2 g_1(V^1 x)),
\end{equation}
with smooth odd\footnote{For brevity we present the full proof in the case of odd activation functions. The argument for the general case (i.e., when only assuming that activation functions are centered w.r.t. Gaussian distribution) is similar, but requires more tedious estimates.} functions $f_1,f_2,g_1, g_2$. We assume that the rows of \(W^1, W^2, V^1, V^2\) are mean-zero i.i.d. samples \((w^1_i, w^2_i, v^1_i, v^2_i) \sim (w^1, w^2, v^1, v^2)\), such that \(C_1 \coloneqq \E w^1(w^1)^{\top}, C_2 \coloneqq \E w^2(w^2)^{\top}, \tilde C_1 \coloneqq \E v^1(v^1)^{\top}\), and \(\tilde C_1 \coloneqq \E v^2(v^2)^{\top}\). Let \(\check C_1 = \E w^1 (v^1)^{\top}\) and \(\check C_2 =\E w^2 (v^2)^{\top}\). Let \((N_1, \wt N_1)\) be a zero-mean jointly Gaussian random variables:
\begin{equation}
    (N_1, \tilde N_1) \sim \cN \left(0,  \begin{pmatrix} \Tr(C_1) & \Tr(\check C_1)\\ \Tr(\check C_1) & \Tr(\tilde C_1) \end{pmatrix} \right),
\end{equation}
and define
\begin{equation}
\begin{aligned}
    &\Phi_1 \coloneqq \E f_1(W^1 x) g_1(V^1x)^\top, \\
    &\Phi_1^{\mathrm{lin}} = (\E f_1'(N_1))(\E g_1'(\wt N_1)) W^1 (V^1)^\top + [\E f_1(N_1)g_1(\wt N_1)-(\E f_1'(N_1))(\E g_1'(\wt N_1))(\E N_1 \wt N_1)] I. \\
\end{aligned}
\end{equation}
Similarly, we define \(\Omega_1, \Omega_1^{\mathrm{lin}}\) with \(V^1, g_1, \wt N_1\) replaced by \(W^1, f_1, N_1\) and we define \(\Psi_1, \Psi_1^{\mathrm{lin}}\) with \(W^1, f_1,  N_1\) replaced by \(V^1, g_1, \wt N_1\) (see~\Cref{def: linearized_covs}).
Next, let \((N_2, \wt N_2)\) be a zero-mean jointly Gaussian random variables:
\begin{equation}
    (N_2, \tilde N_2) \sim \cN \left(0,  \begin{pmatrix} \Tr(C_2 \Omega_1^{\mathrm{lin}}) & \Tr (\check C_2 \Phi_1^{\mathrm{lin}})\\ \Tr (\check C_2 \Phi_1^{\mathrm{lin}}) & \Tr(\tilde C_2 \Psi_1^{\mathrm{lin}}) \end{pmatrix} \right),
\end{equation}
and define
\begin{equation}
\begin{aligned}
    &\Phi_2 \coloneqq \E f_2(W^2f_1(W^1 x)) g_2(V^2g_1(V^1x))^\top, \\
    &\Phi_2^{\mathrm{lin}} = (\E f_2'(N_2))(\E g_2'(\wt N_2)) W^2 \Phi_1^{\mathrm{lin}}(V^2)^\top + [\E f_2(N_2)g_2(\wt N_2)-(\E f_2'(N_2))(\E g_2'(\wt N_2))(\E N_2 \wt N_2)] I, \\
\end{aligned}
\end{equation}
and, again, similarly \(\Omega_2, \Omega_2^{\mathrm{lin}}, \Psi_2,\) and  \(\Psi_2^{\mathrm{lin}}\).
\begin{theorem}
\label{thm:lin_app}
    We have that \(\norm{\Phi_2 - \Phi_2^{\mathrm{lin}}}_F \prec 1\).
\end{theorem}
We split the proof into the following lemmas:
\begin{lemma}[Diagonal entries of \(\Phi_2\)]
\label{lem:diag}
For possibly correlated vectors \(u, z\), we have the bound
    \begin{equation}
        \abs*{\E f_2(u^{\top} f_1(W^1 x)) g_2(z^{\top} g_1(V^1 x)) - \E f_2(N_2)g_2(\tilde N_2)} \prec d^{-1/2}.
    \end{equation}
\end{lemma}

\begin{lemma}[Off-diagonal entries of \(\Phi_2\)]
\label{lem:off_diag}
If \(u\) and \(z\) are independent, we have
    \begin{equation}
        \abs*{\E f_2(u^{\top} f_1(W^1x)) g_2(z^{\top} g_1(V^1x)) - (\E f_2'(N_1)) (\E g_2'(N_2)) u^\top \Phi_1^{\mathrm{lin}} z} \prec d^{-1}.
    \end{equation}
\end{lemma}


\begin{proof}[Proof of~\cref{thm:lin_app}]
    The proof follows from~\Cref{lem:diag,lem:off_diag} upon summation over all entries.
\end{proof}

\begin{proof}[Proof of~\Cref{theo lin}]
The proof follows from~\Cref{prop:lin_1_layer} and~\Cref{thm:lin_app}. Note that results for \(\Omega_i\) and \(\Psi_i\) can be obtained using aforementioned results only for \(f_i, W^i\) and only for \(g_i, V^i\) respectively.
\end{proof}

\subsection{Proof of~\cref{lem:off_diag}}
For simplicity of notation, we omit indices in \(W^1, V^1\) and write \(W, V\) instead. We begin with showing that
    \begin{equation}
        \E f_2(u^{\top} f_1(Wx)) g_2(z^{\top} g_1(Vx)) = \E \left[f_2' (u^\top f_1(Wx))\right] \E \left[g_2' (z^\top g_1(Vx))\right]  u^\top \Phi_1 z + O(d^{-1}),
    \end{equation}
for \emph{independent} random vectors \(u\) and \(z\).
Recall that 
\begin{equation}
    f_2(u^\top f_1(Wx)) = \sum_{p \geq 1} \frac{1}{p!} I_p \left(\underbrace{\E \sum_{\pi \vdash[p]} f_2^{(\Card{\pi})} (u^\top f_1(Wx)) \wt \bigotimes_{B \in \pi} \left(\sum u_k f_1^{(\Card{B})}(w_i^\top x) w_i^{\otimes \Card{B}}\right)}_{\E D^p f_2(u^\top f_1(Wx))}\right),
\end{equation}
and that \(\E f_2(u^{\top} f_1(Wx))g_2(z^{\top} g_1(Vx)) = \sum_{p \geq 1} \frac{1}{p!} \braket{\E D^p f_2(u^{\top} f_1(Wx)), \E D^p g_2(z^{\top} g_1(Vx))}. \)

Let \(f_2^{p} \coloneqq \E f_2^{(p)}(u^\top f_1(Wx))\) and \(g_2^{p} \coloneqq \E g_2^{(p)}(z^\top g_1(Vx))\). 
\Cref{lemma:weak_corr} implies that
    \begin{equation}
    \begin{aligned}
        &\Bigl<\E \left[f_2^{(\Card{\pi})}(u^\top f_1(W^1 x)) \wt \bigotimes_{B \in \pi} \left(
        \sum_k u_k f_1^{(\Card{B})} (w_i^\top x) w_i^{\otimes \Card{B}}\right)\right], \\
        &\quad \E \left[g_2^{(\Card{\pi'})}(z^\top g_1(V^1 x)) \wt \bigotimes_{B \in \pi'} \left(
        \sum_k u_k f_1^{(\Card{B})} (w_i^\top x) w_i^{\otimes \Card{B}}\right)\right]\Bigr> \\
        & \quad = f_2^{\Card{\pi}} g_2^{\Card{\pi'}} 
        \braket*{\wt\bigotimes_{B \in \pi} \left(
        \sum_k u_k f_1^{(\Card{B})} (w_i^\top x) w_i^{\otimes \Card{B}}\right), \wt\bigotimes_{B' \in \pi'} \left(
        \sum_k z_k g_1^{(\Card{B'})} (v_i^\top x) v_i^{\otimes \Card{B'}}\right)} + O(d^{-1}).
    \end{aligned}
    \end{equation}

Using it, we can write (denoting \(f_{1i}^p \coloneqq \E f_1^{(p)}(w_i^\top x)\) and \(g_{1i}^p \coloneqq \E g_1^{(p)}(v_i^\top x)\))
\begin{equation}
\begin{aligned}
    & \braket{\E D^p f_2(u^{\top} f_1(Wx)), \E D^p g_2(z^{\top} g_1(Vx))} \\
    & \quad = \sum_{\substack{\pi, \pi' \vdash [p]}} \sum_{\substack{i_1, \ldots, i_{\lvert \pi \rvert} \\
    j_1, \ldots, j_{\rvert \pi' \rvert}}} f_2^{\Card \pi} g_2^{\Card {\pi'}} \prod_{k = 1}^{\Card{\pi}} \left( f_{1i}^{b(k)} u_{i_k} \right) \prod_{k = 1}^{\Card{\pi'}} \left( g_{1i}^{b'(k)} z_{j_k} \right) \frac{1}{p!} \sum_{\sigma \in S_p} \prod_{q=1}^p \braket{w_{i_{\pi(q)}}, v_{j_{\pi'(\sigma(q))}}} + O(d^{-1}),
\end{aligned}
\end{equation}
where \(b(k)\) and \(b'(k)\) denote the size of \(k\)th block in \(\pi\) and \(\pi'\) respectively.
The term \(\pi = \pi' = \{[p]\}\) corresponds to 
\begin{equation}
    \sum_{i, j} \E f_2' (u^\top f_1(Wx)) \E g_2'(z^\top g_1(Vx)) f_1^{(p)}(w_i^\top x)g_1^{(p)}(v_j^\top x)u_i z_j \braket{w_{i}, v_{j}}^p,
\end{equation}
which, after summing over \(p \geq 1\) is equal to
\begin{equation}
    \E \left[f_2' (u^\top f_1(Wx))\right] \E \left[g_2' (z^\top g_1(Vx))\right] u^\top \E \left[f_1(W^\top x)g_1(V^\top x)\right]z.
\end{equation}
Therefore, it remains to show that all the other terms contribute in total \(O(d^{-1})\). Note that \(f_2^{\Card{\pi}} = O(1)\) and same for other derivatives.
\begin{lemma}
    Fix \(\pi, \pi' \vdash [p]\). Then,
    \begin{equation}
       \sum_{\substack{i_1, \ldots, i_{\lvert \pi \rvert} \\
    j_1, \ldots, j_{\rvert \pi' \rvert}}}  \prod_{k = 1}^{\Card{\pi}}  u_{i_k} \prod_{k = 1}^{\Card{\pi'}}  z_{j_k} \prod_{q=1}^p \braket{w_{i_{\pi(q)}}, v_{j_{\pi'(q)}}} \prec d^{\frac{1}{2}(\min(\Card{\pi}, \Card{\pi'}) - p)}.
    \end{equation}
    
\end{lemma}
\begin{proof}
    Without loss of generality, it is enough to show upper bound \(\prec d^{\frac{1}{2}(\Card{\pi'} - p)}\). By separating terms that depend on \(i_k\) for \(k \in [\Card{\pi}]\), we can rewrite
    \begin{equation}
        \sum_{\substack{i_1, \ldots, i_{\lvert \pi \rvert} \\
    j_1, \ldots, j_{\rvert \pi' \rvert}}}  \prod_{k = 1}^{\Card{\pi}}  u_{i_k} \prod_{k = 1}^{\Card{\pi'}}  z_{j_k} \prod_{q=1}^p \braket{w_{i_{\pi(q)}}, v_{j_{\pi'(q)}}} = \sum_{j_1, \ldots, j_{\Card{\pi'}}} \prod_{k=1}^{\Card{\pi'}} z_{j_k} \prod_{k=1}^{\Card{\pi}} \left(\sum_{i_k} u_{i_k} \prod_{q: \pi(q) = k} \braket{w_{i_k}, v_{j_{\pi'(q)}}}\right).
    \end{equation}
    Note that \(\sum_{i_k} u_{i_k} \prod_{q: \pi(q) = k} \braket{w_{i_k}, v_{j_{\pi'(q)}}} \prec d^{-b(k)/2}\), where \(b(k)\) denotes the size of \(k\)-th block of \(\pi\). Since \(\sum_{k} b(k) = p\), we bound the total expression by \(d^{\Card{\pi'} - \Card{\pi'}/2 - p/2} = d^{\frac{1}{2}(\Card{\pi'} - p)}\).
\end{proof}
Since we assume that \(f_1, g_1\) are odd, terms with \(\Card{\pi} = p - 1\) are equal to zero, and same with \(\Card{\pi'} = p - 1\). When \(\min(\Card{\pi}, \Card{\pi'}) \leq p - 2\), the previous lemma implies \(O(d^{-1})\) total contribution. Therefore, the only case left is with \(\pi = \pi' = \{\{1\}, \{2\}, \ldots, \{p\}\}\). However, in this case it is easy to see that the final contribution is \(O(d^{-p/2})\).

Next, note that perturbation analysis implies that \(\E f_2'(u^\top f_1(W^1 x)) = \E f_2'(N_2) + O(d^{-1/2})\), same for \(g_2\). Finally, using that \(\norm{\Phi_1 - \Phi_1^{\mathrm{lin}}}_F \prec 1,\) we obtain that
\begin{equation}
    \abs*{\E f_2(u^{\top} f_1(W^1x)) g_2(z^{\top} g_1(V^1x)) - (\E f_2' (N_2)) (\E g_2' (\tilde N_2))  u^\top \Phi_1^{\mathrm{lin}} z} \prec d^{-1},
\end{equation}
which finishes the proof.

\subsection{Proof of~\Cref{lem:diag}}
For convenience we restate several concentration results that follow from~\Cref{corr rainbow}.
\begin{lemma}
\label{lem:conc_inequalities}
    Let \(w_1, \ldots, w_d\) be a collection of independent random vectors, such that for all \(\norm{w_i}_{\psi_2} = O(d^{-1/2})\) and \(\norm{w_i} = O(1)\) for all \(i \in [d]\).
    Then 
    \begin{enumerate}[label=(\roman*)]
        \item \(\sum_i \braket{w_1, w_i} \prec 1\),
        \item \(\sum_{ij} \braket{w_1, w_i} \braket{w_1, w_j} \prec 1\).
    \end{enumerate}
\end{lemma}
\begin{proof}
    Let \(X = \sum_{i \geq 2} \braket{w_1, w_i}\). Note that \(\norm{X}_{\psi_2} = O(1)\), which, together with \(\norm{w_1} = O(1)\) implies (i).

    For \((ii)\), note that \(\sum_{ij} \braket{w_1, w_i} \braket{w_1, w_j} = \left(\sum_i \braket{w_1, w_i}\right)^2 \prec 1\) using \((i)\).
\end{proof}
Let \(F_1 = u^\top f_1(W^1 x)\) and \(F_2 = z^\top g_1(V^1 x)\), where \(u, z\) may be correlated. For simplicity we omit indices in \(f_1, g_1, W^1\), and \(V^1\).
Using Wiener chaos expansion, we obtain
\begin{equation}
    F_1 = u^\top f(Wx) = \sum_{p \geq 1} I_p\left(\frac{\E D^p F_1}{p!}\right) = \sum_{p \text{ odd}} I_p \left(\sum_i \frac{u_i w_i^{\otimes p} \E f^{(p)} (w_i^\top x)}{p!}\right),
\end{equation}
where the last equality uses the fact that \(\E f^{(p)} (w_i^\top x) = 0\) for even \(p\). Similarly, we can write
\begin{equation}
    F_2 = z^\top g(Vx) = \sum_{p \text{ odd}} I_p \left(\sum_i \frac{z_i v_i^{\otimes p} \E g^{(p)} (v_i^\top x)}{p!}\right),
\end{equation}
and denote \(f_i^p \coloneqq \E f^{(p)} (w_i^\top x), g_i^p \coloneqq \E g^{(p)} (v_i^\top x)\). Next, we compute
\begin{equation}
    D F_1 = \sum_{\text{odd } p} p I_{p -1} \left(\sum_i \frac{u_i w_i^{\otimes p} f_i^p}{p!} \right) \quad \text{and} \quad -D L^{-1} F_2 = \sum_{\text{odd } q} I_{q -1} \left(\sum_i \frac{z_i v_i^{\otimes q} g_i^q}{q!} \right)
\end{equation}
\begin{lemma}
    \begin{equation}
        \E (\braket{D F_1, -D L^{-1} F_2} - \E F_1 F_2)^2 = O(d^{-1}).
    \end{equation}
\end{lemma}
\begin{proof}
    Since \(I_{p - 1} (w_i^{\otimes p}) = I_{p - 1}(w_i^{\otimes p - 1}) w_i\), we can write
    \begin{equation}
        \braket{D F_1, -DL^{-1} F_2} = \sum_{\substack{\text{odd } p \\ {\text{odd } q}}} c_{pq} \sum_{i, j} \braket{w_i, v_j} u_i z_j f_i^p g_j^q I_{p - 1}(w_i^{\otimes p - 1}) I_{q - 1}(w_j^{\otimes q - 1}).
    \end{equation}
    Using (\cite{nourdin2012normal}, Theorem 2.7.10), we obtain
    \begin{equation}
    \begin{aligned}
        & I_{p - 1}(w_i^{\otimes p - 1} I_{q - 1}(w_j^\otimes q - 1) \\
        & \quad = \sum_{r = 0}^{p \land q - 1} \braket{w_i, v_j}^r c_{rpq} I_{p + q - 2(r + 1)} (w_i^{\otimes p - 1 - r} \wt \otimes v_j^{\otimes q - 1 - r}) \\
        & \quad = \sum_{\substack{s = |p - q| \\ 2 \text{ divides } (s - |p - q|)}}^{p + q - 2} c'_{spq} \braket{w_i, v_j} ^{(p + q - 2 - s) / 2} I_s(w_i^{\otimes p - q + s} \wt \otimes v_j^{q - p + s}).
    \end{aligned}
    \end{equation}

Therefore, we obtain
\begin{equation}
    \braket{DF_1, -DL^{-1}F_2} = \sum_{s \geq 0} \sum_{\substack{\abs{p -q} \leq s \\ 2 \text{ divides } (s - \abs{p - q}) \\ p \land q \geq 1 + (s - \abs{p - q}) / 2}} \tilde c_{r,p,q}\sum_{i, j}\braket{w_i, v_j}^{(p+q - s)/2} u_i z_j f^p_i f^q_j 
    I_s(w_i^{\otimes (s + p - q) / 2} \wt \otimes v_j^{\otimes (s + q - p) / 2}).
\end{equation}
The term \(s = 0\) corresponds to \(\E F_1 F_2\). Since \(p\) and \(q\) must be odd in the non-zero terms of the sum, we obtain that \(s\) must be even. For \(a \coloneqq (p+q-s)/2\), the \(s\)-th multiplie integral \(I_s\) can be rewritten as follows:
\begin{equation}
    I_s\left(\sum_{a \geq 1} \sum_{i, j} \braket{w_i, v_j} u_i z_j T_{ij}^s\right),
\end{equation}
where \(T_{ij}^s\) is a \(s\)-dimensional tensor, consisting of a sum of inner products of \(w_i\) and \(v_j\), also containing combinatorial terms, and products of expectations of derivatives of \(f, g\). We can write
\begin{equation}
    \E \left(\braket{D F_1, -DL^{-1} F_2} - \E F_1 F_2\right)^2 = \sum_{s \geq 2} \E I_s\left(\sum_{a \geq 1} \sum_{i,j} \braket{w_i, v_j} u_i z_j T_{ij}^s\right)^2.
\end{equation}
Fix \(s \geq 2\) and observe that 
\begin{equation}
\label{eq:fixed_s_expansion}
    \E I_s \left(\sum_{a \geq 1} \sum_{i,j} \braket{w_i, v_j} u_i z_j T_{ij}^s\right)^2 = 
    \sum_{a,a' \geq 1} \sum_{\substack{i,j\\i', j'}} \braket{w_i, v_j}^a \braket{w_{i'}, v_{j'}}^{a'} u_i u_{i'} z_j z_{j'} \braket{T_{ij}^s, T_{i'j'}^s},
\end{equation}
and note that for some constant \(C > 0\) (depending on combinatorial terms, and products of expectations of derivatives of \(f\))
\(\braket{T_{ij}^s, T_{i' j'}^s}\) can be upper bounded by
\begin{equation}
    \braket{T_{ij}^s, T_{i'j'}^s} \leq C(\braket{w_i, w_{i'}} + \braket{w_i, v_{j'}} + \braket{v_j, w_{i'}} + \braket{v_j, v_{j'}})^s.
\end{equation}
We analyze each term of the summand in~\Cref{eq:fixed_s_expansion} depending on \(a, a', i, i', j, j'\).
Let \(N = \lvert\{i, i', j, j'\}\rvert\), the number of distinct indices among \(i, i', j, j'\).
Since entries of \(u\) and \(z\) are \(O(d^{-1/2})\), we get that in total the term \(u_i u_{i'} z_j z_{j'}\) contributes \(O(d^{-2})\).

\paragraph{Case \(N = 1\)}
Here, since there are only \(d\) such terms in total, we immediately obtain an \(O(d^{-1})\) upper bound.
\paragraph{Case \(N = 2\)}
There are \(O(d^2)\) such terms. It must be that either (i) \(i \neq j\) or (ii) \(i' \neq j'\) or (iii) both \(i = j\) and \(i' = j'\).
In the latter case, we obtain bound \(O(d^{-1})\) since \(s \geq 2\), and thus \(\braket{T_{ij}^s, T_{i'j'}^s} = O(d^{-1})\).
Otherwise, without loss of generality, assume that \(i = i' = j'\) and \(a = 1\). Here,~\Cref{lem:conc_inequalities} (i) implies that the summand is \(\prec 1/d\).
\paragraph{Case \(N = 3\)}
If \(i = j\), note that \(\braket{T_{ij}^s, T_{i'j'}^s} \prec d^{-1}\) and we need to show that \(\sum_{i \neq i' \neq j'} \braket{w_{i'}, v_{j'}}^a \prec d^2\).
When \(a \geq 2\) this follows from asymptotic orthogonality (\(\braket{w_{i'}, v_{j'}} \prec d^{-1/2}\) for \(i' \neq j'\), see~\Cref{corr rainbow}), and otherwise from~\Cref{lem:conc_inequalities} (i).
If \(i = i'\), we need to show that 
\begin{equation}
    \sum_{i j j'} \braket{w_i, v_j}^a \braket{w_{i'}, v_{j'}}^{a'} d^{-2} \prec \frac{1}{d}.
\end{equation}
When \(a \geq 2\) and \(a' \geq 2\), this follows from asymptotic orthogonality.
When \(a = 1\) and \(a' \geq 2\) (or vice versa), this follows from~\Cref{lem:conc_inequalities} (i).
Finally, when \(a = 1\) and \(a' = 1\), this follows from~\Cref{lem:conc_inequalities} (ii).
The remaining cases are identical to the covered ones.
\paragraph{Case \(N = 4\)}
When \(a \geq 2\) and \(a' \geq 2\), the result follows trivially.
When \(a = 1\) and \(a' \geq 2\), the result follows from~\Cref{lem:conc_inequalities} (i).
When \(a = a' = 1\), the result follows again from~\Cref{lem:conc_inequalities} (i) and noticing that 
\begin{equation}
    \frac{1}{d^2} \sum_{\substack{i \neq j \\ i' \neq j'}} \braket{w_i, v_j} \braket{w_{i'}, v_{j'}} = \left(\frac{1}{d} \sum_{i \neq j} \braket{w_i, v_j}\right)^2 \prec 1.
\end{equation}
\end{proof}
Next, using~\Cref{thm:multi-clt} for \(h(F_1, F_2) = f_2(F_1)g_2(F_2)\), we obtain that 
\begin{equation}
    \abs*{\E f_2(u^\top f_1(W^1 x)) g_2(z^\top g_1(V^1 x)) - \E f_2(G_1) g_2(G_2)} \prec d^{-1/2},
\end{equation}
where \((G_1, G_2)\) is a jointly Gaussian random vector:
    \begin{equation}
        (G_1, G_2) \sim \cN \left(0,  \begin{pmatrix} \Tr \left[u u^\top \Omega_1 \right] & \Tr \left[u z^\top \Phi_1 \right]\\ \Tr \left[u z^\top \Phi_1 \right] & \Tr \left[z z^\top \Psi_1 \right] \end{pmatrix} \right).
\end{equation}
Finally, using that \(\norm{\Omega_1 - \Omega_1^{\mathrm{lin}}}_F \prec 1\) (same for \(\Phi_1, \Psi_1\)) and perturbation analysis, we obtain that 
\begin{equation}
    \abs*{\E f_2(u^\top f_1(W^1 x)) g_2(z^\top g_1(V^1 x)) - \E f_2(N_1) g_2(N_2)} \prec d^{-1/2}.
\end{equation}

\subsection{Extending to \(L \geq 3\)}
A natural question is to ask whether the same technique can be applied for a deeper networks. One possible direction is to apply~\Cref{thm:multi-clt} for \(d\)-dimensional vector \((F_1, \ldots F_d) \coloneqq (u_1^\top f_1(W^1 x), \ldots, u_d^\top f_1(W^1 x))\), to approximate it by a Gaussian random vector \((N_1, \ldots, N_d)\). Then, for example, the diagonal entries of \(\Omega_3\) can be written as \(h(F_1, \ldots, F_d) = f_3(\sum_{k} u_k f_2(F_k))^2\). If it is possible to derive that \(f_3(\sum_{k} u_k f_2(F_k))^2 = f_3(\sum_k u_k f_2(N_k))\), then the problem is reduced to the 2 layered case, which can be treated as before.

However, it seems hard to apply~\Cref{thm:multi-clt} to the \(d\)-dimensional vector, since this requires a much more careful error analysis. Recall that we only applied~\Cref{thm:multi-clt} to \(2\)-dimensional vectors. We leave the extension to \(L, \wt L \geq 3\) as an interesting open question.

%% file: sections/appendix/phenomenology.tex
\subsection{Details of Fig.\,\ref{fig:PL}}

\paragraph{Target} We consider a two-layer structured RF teacher, with feature map
\begin{align}
    \varphi_*(x)=\mathrm{tanh}\left(
    W_*x
    \right)
\end{align}
where the weight $W_*=Z_*\Tilde{C}_1^{\frac{1}{2}} \in\mathbb{R}^{d\times d}$ has covariance
\begin{align}
    \Tilde{C}_1=\mathrm{diag}(\{k^{-0.3}\}_{1\le k\le d}).
\end{align}

\paragraph{Student} We consider the task of learning this target with a four-layer RF student, with feature map
\begin{align}
    \varphi(x)=\tanh W_3(\tanh\left(
    W_2\tanh(W_1x)
    \right))
\end{align}
where, in order to introduce inter-layer and target/student weight correlations, we considered $W_2=W_1$, with
\begin{align}
    W_1=\sfrac{1}{2}Z_1\mathrm{diag}(\{k^{-\sfrac{\gamma}{2}}\}_{1\le k\le d})+\sfrac{1}{2}W_*,
\end{align}
for $\gamma\in\{0.0,0.2,0.5,0.8\}$. In other words, the covariance $C_1$ of $W_1,W_2$ is a sum of two power laws with decay $\gamma$ and $0-3$. Finally, in order to introduce another form of correlation, we chose
\begin{align}
    W_3=Z_3 C_3^{\sfrac{1}{2}}
\end{align}
where the covariance $C_3$ depends on the previous weights as
\begin{align}
C_3=(W_1W_1^\top +\sfrac{1}{2}\mathbb{I}_d)^{-1}.
\end{align}

\subsection{MNIST Experiments}
\label{app:real_data}
\paragraph{Data set}
We use the MNIST data set which we normalize by pixel-wise centering and global scaling to ensure unit variance. For each normalized image $x_i\in\R^{784}$ we define a label
\[y_i := \begin{cases} 1,&  \text{ if }x_i\text{ is an even digit,}  \\ -1,& \text{ if }x_i\text{ is an odd digit.}\end{cases}
\]
We split the data set into four parts:
\begin{itemize}
\item[$10\%$] Test data $I_\mathrm{test}$ 
\item[$25\%$] Training data for the Adam optimizer $I_\mathrm{Adam}$,
\item[$25\%$] Training data for regression $I_\mathrm{reg}$,
\item[$40\%$] Data for approximating the (empirical) population covariance $I_\mathrm{emp}$.
\end{itemize}

\paragraph{Neural network}
We then train a simple neural network of the form 
\begin{equation}
    x\mapsto \theta^\top\varphi(x), \quad \varphi(x):=\theta^\top\operatorname{relu}(W_2 \operatorname{relu}( W_1 x)), \quad W_1\in\R^{2352\times 784},\quad W_2\in\R^{2352\times 2352}, \quad \theta\in\R^{2352}
\end{equation}
using the Adam optimizer over $120$ epochs with a batch size of $128$ using only the $I_\mathrm{Adam}$ split. During training we save the \emph{feature maps} $\varphi_t$ at various time steps $t$ in order to study the training dynamics. 

\paragraph{Feature ridge regression} We then perform a ridge regression task using the features $\varphi_t(x_i)$ by minimizing
\begin{equation}
    \theta(t,\lambda,I):=\arg\min_\theta \Bigl(\frac{1}{\abs{I}}\sum_{i\in I} (y_i- \theta^\top\varphi_t(x_i))^2 + \lambda \norm{\theta}^2\Bigr)
\end{equation}
for various random subsets $I\subset I_\mathrm{reg}$, and $I_\mathrm{test}$, empirically estimate the \emph{generalization error} 
\begin{equation}
    \mathcal E_\mathrm{gen}(t,\lambda,I)^2 := \frac{1}{\abs{I_\mathrm{test}}} \sum_{i\in I_\mathrm{test}} (y_i - \theta(t,\lambda,I)^\top\varphi_t(x_i))^2.
\end{equation}

\paragraph{Deterministic equivalent}
In order to compare $\mathcal E_\mathrm{gen}(t)$ with the theoretical prediction from~\Cref{thm genRMT informal}, we need to determine the covariance of the features $\phi_t$ as well as the label-feature covariance and the label variance. To do so, we use the $I_\mathrm{emp}$ part of the data to empirically estimate 
\begin{equation}
    \Omega_t:=\frac{1}{\abs{I_\mathrm{emp}}}\sum_{i\in I_\mathrm{emp}} \varphi_t(x_i)\varphi_t(x_i)^\top\in\R^{2352\times2352},\quad \psi_t:=\frac{1}{\abs{I_\mathrm{emp}}}\sum_{i\in I_\mathrm{emp}} \varphi_t(x_i)y_i\in\R^{2352}, \quad \sigma^2 := \frac{1}{\abs{I_\mathrm{emp}}}\sum_{i\in I_\mathrm{emp}} y_i^2\in \R 
\end{equation}
and note that we expect this to be a reasonable approximation since $\abs{I_\mathrm{emp}}=27805\gg 2352$. Using these we have the formula 
    \begin{equation}\label{eq Egen rmt2}
            \mathcal E_\mathrm{gen}^\mathrm{rmt}(t,\lambda,n): = \frac{ \sigma^2 -  n\lambda m_t\psi_t^\top (M_t+\lambda M_t^2)\psi_t }{1-n(m_t\lambda)^2\Tr \Omega_t M_t\Omega_t M_t}
    \end{equation}
analogous to~\Cref{eq Egen rmt}, where $m_t,M_t=m_t(\lambda,n),M_t(\lambda,n)$ are the solution to 
\begin{equation}
    \frac{1}{m_t(\lambda,n)} = \lambda + \Tr \Omega_t (1 + n m_t(\lambda,n)\Omega_t )^{-1}, \quad M_t(\lambda,n) := (\lambda + \lambda n m(\lambda,n)\Omega_t )^{-1}.
\end{equation}

We observe in~\Cref{fig:real_emp} that $\mathcal E^\mathrm{rmt}_\mathrm{gen}$ is indeed an excellent approximation for $\mathcal E_\mathrm{gen}$ throughout the training and for various choices of regularization. In~\Cref{fig:real_det} we depict the interesting dynamics of the learning curves throughout the training process with a significant shift of the interpolation threshold to the left. 
\begin{figure}[ht]
    \centering
    \includegraphics[width=\linewidth]{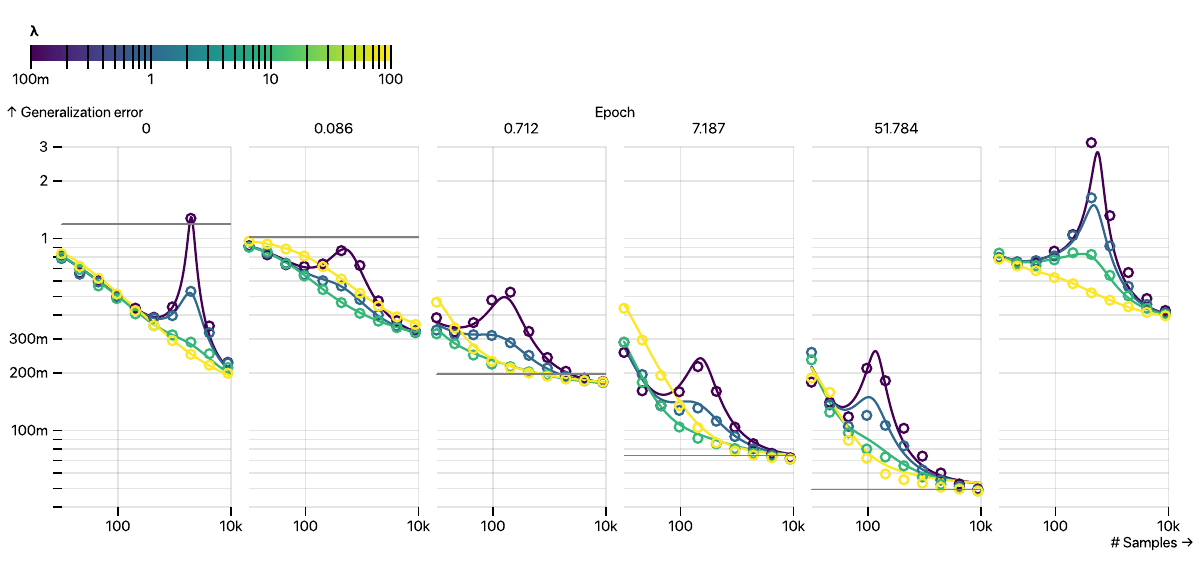}
    \caption{Plot of $\mathcal E^\mathrm{rmt}_\mathrm{gen}$, $\mathcal E_\mathrm{gen}$ for various regularization parameters $\lambda$ and time steps $t$ (in ``epoch.step'' format). The horizontal lines represent the generalization error of the neural network, the curves $\mathcal E^\mathrm{rmt}_\mathrm{gen}$ and the dots $\mathcal E_\mathrm{gen}$. The last pane contains a linear regression model for the sake of comparison. Interestingly, for this particular case already the random feature model outperforms linear regression.}
    \label{fig:real_emp}
\end{figure}
\begin{figure}[ht]
    \centering
    \includegraphics[width=\linewidth]{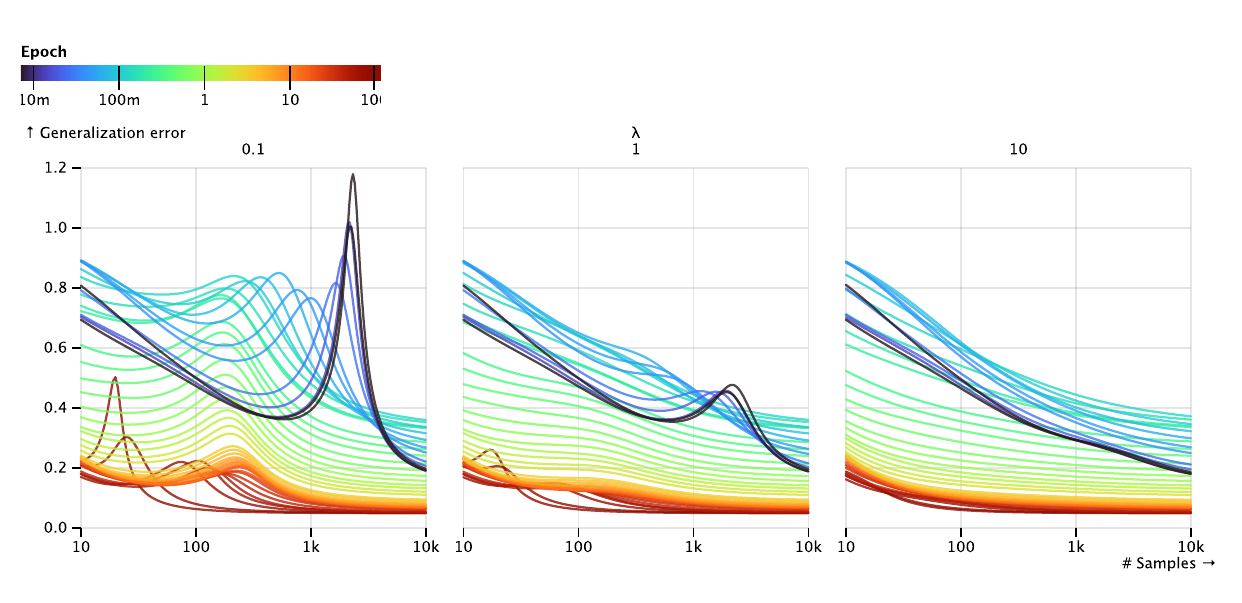}
    \caption{Dynamics of $\mathcal E^\mathrm{rmt}_\mathrm{gen}$ throughout the training}
    \label{fig:real_det}
\end{figure}

\paragraph{Optimal regularization}
So far we have focused on fixed regularization parameters. Using the deterministic equivalent we can also find the optimal regularization parameter \begin{equation}
    \lambda_\mathrm{opt}(t,n) := \arg\min_\lambda \mathcal E_\mathrm{gen}^\mathrm{rmt}(t,\lambda,n) 
\end{equation} 
for each sample complexity $n$ and time $t$ by simply one-dimensional minimization. In~\Cref{fig:real_opt} we show the corresponding results. Interestingly ridge regression initially performs worse than the random feature regression also at optimal regularization. Then in the initial phase of training the performance of feature regression deteriorates before improving way beyond the initialization performance.   
\begin{figure}
    \centering
    \includegraphics[width=.6\linewidth]{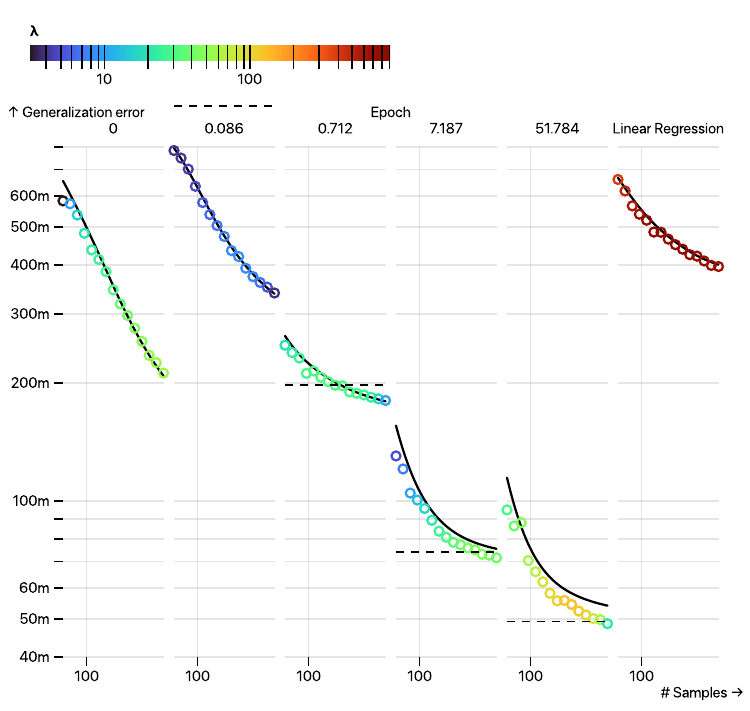}
    \includegraphics[width=.37\linewidth]{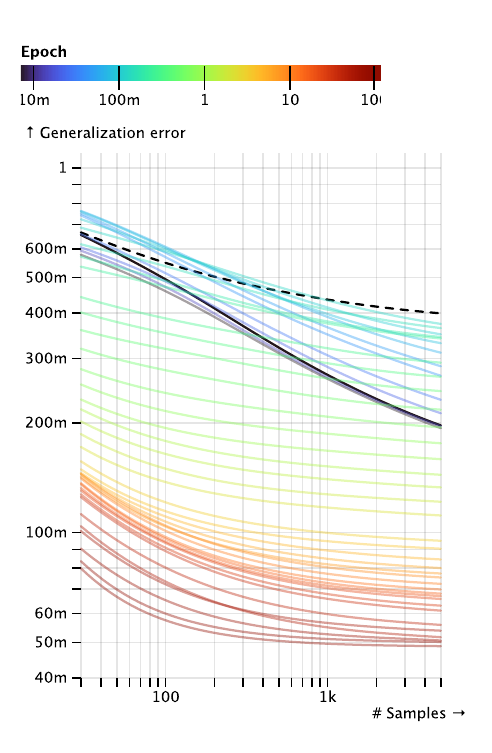}
    \caption{The left pane shows $\mathcal E_\mathrm{gen}$ and $\mathcal E_\mathrm{gen}^\mathrm{rmt}$ throughout the training process at \emph{optimal} regularization $\lambda_\mathrm{opt}$. The colour of the dots encodes the value of $\lambda_\mathrm{opt}$, while the dashed lines represent the generalization error of the neural network. The right pane shows the dynamics of $\mathcal E_\mathrm{gen}^\mathrm{rmt}$ throughout the training process, compared with linear regression.}
    \label{fig:real_opt}
\end{figure}

\subsection{Synthetic MNIST experiments}\label{synth MNIST}
We carried out similar experiments for synthetic data in order to empirically study the effect of population covariance linearization. 

\paragraph{Data}
We generate Gaussian random vectors of zero mean and variance matching the variance of the normalized MNIST images described above. The synthetic labels are generated by a one hidden layer random feature network 
\[ \varphi_\ast(x):= \theta_\ast \tanh(W_\ast x), \quad W_\ast\in\R^{800\times 784}, \theta_\ast\in\R^{800} \]
for fixed but random $W_\ast,\theta_\ast$. 

\paragraph{Neural network}
We again train a simple neural network of the form 
\begin{equation}
    x\mapsto \theta^\top\varphi(x), \quad \varphi(x):=\theta^\top\operatorname{relu}(W_2 \operatorname{relu}( W_1 x)), \quad W_1\in\R^{800\times 784},\quad W_2\in\R^{700\times 800}, \quad \theta\in\R^{700}
\end{equation}
using the Adam optimizer over $50$ epochs with a batch size of $128$ using $40\ 000$ samples. During training we save the \emph{feature maps} $\varphi_t$ at various time steps $t$ in order to study the training dynamics. 

\paragraph{Feature ridge regression} We perform feature ridge regression on the trained features $\varphi_t$ exactly as described above. The fact that the labels are now generated by a feature model now enables us to test the effect of population covariance linearization. In~\Cref{fig:art-lin} we observe that for random features the linearized deterministic equivalent is an excellent approximation for the empirically observed feature ridge regression error. However, during training the prediction deteriorates. We suspect that this effect is due to outlying eigevalues of the weight matrices which increasingly violate~\Cref{corr rainbow}.  

\begin{figure}
    \centering
    \includegraphics[width=\linewidth]{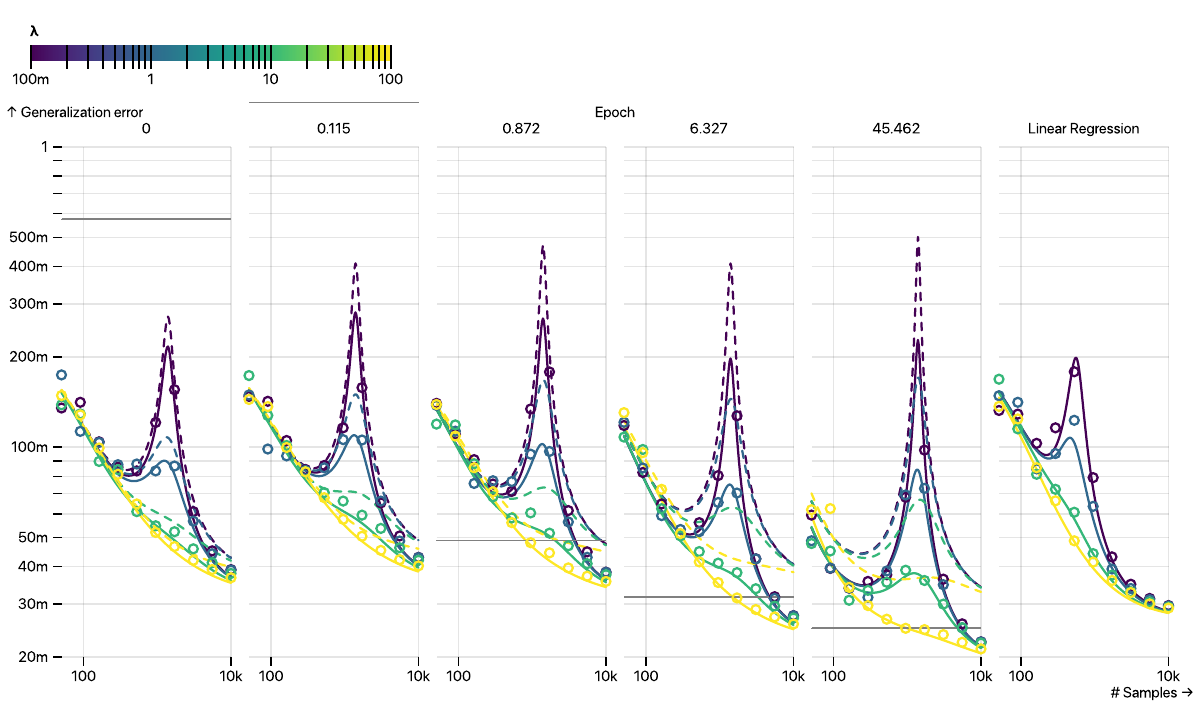}
    \caption{The solid and dashed lines represent $\mathcal E_\mathrm{gen}^\mathrm{rmt}$ using the empirical population covariances and the linearized population covariances, respectively. The dots represent the empirical $\mathcal E_\mathrm{gen}$ while the horizontal line show the test error of the neural network during training. The right-most pane shows linear regression for comparison.}
    \label{fig:art-lin}
\end{figure}